\documentclass[12pt,nosubfloats]{alt2024} 

\title[Computation with Sequences in the Brain]{Computation with Sequences of Assemblies in a Model of the Brain}
\usepackage{times}
\usepackage{enumitem}
\usepackage{subcaption}
\usepackage{bbm}

\newcommand{\R}{\mathbb{R}}

\newcommand{\E}{\mathbb{E}}

\renewcommand{\vec}[1]{\mathbf{#1}}
\newcommand{\1}{\mathbbm{1}}
\newcommand{\cond}{\, | \,}

\altauthor{%
 \Name{Max Dabagia} \Email{maxdabagia@gatech.edu}\\
 \addr School of Computer Science, Georgia Tech
 \AND
 \Name{Christos H. Papadimitriou} \Email{christos@columbia.edu}\\
 \addr Department of Computer Science, Columbia University
 \AND
 \Name{Santosh S. Vempala} \Email{vempala@gatech.edu}\\
 \addr School of Computer Science, Georgia Tech
}

\begin{document}

\maketitle

\begin{abstract}%
Even as machine learning exceeds human-level performance on many applications, the generality, robustness, and rapidity of the brain's learning capabilities remain unmatched. How cognition arises from neural activity is {\em the} central open question in neuroscience, inextricable from the study of intelligence itself. A simple formal model of neural activity was proposed in \citet{papadimitriou2020brain} and has been subsequently shown, through both mathematical proofs and simulations, to be capable of implementing certain basic cognitive operations via the creation and manipulation of assemblies of neurons. However, many intelligent behaviors rely on the ability to recognize, store, and manipulate temporal {\em sequences} of stimuli and memories (planning, language, navigation, to list a few). Here we show that, in the same model, the precedence of time can be captured naturally through synaptic weights and plasticity, and, as a result, a range of computations on {\em sequences} of assemblies can be carried out.  In particular, repeated presentation of a sequence of stimuli leads to the memorization of the sequence through corresponding neural assemblies: Upon future presentation of any stimulus in the sequence, the corresponding assembly and its subsequent ones will be activated, one after the other, until the end of the sequence.  If the stimulus sequence is presented to two brain areas simultaneously, a scaffolded representation is created, resulting in more efficient memorization and recall, in agreement with cognitive experiments. Finally, we show that any finite state machine can be learned in a similar way, through the presentation of appropriate patterns of sequences.  Through an extension of this mechanism, the model can be shown to be capable of universal computation. We support our analysis with a number of experiments to probe the limits of learning in this model in key ways. Taken together, these results provide a concrete hypothesis for the basis of the brain's remarkable abilities to compute and learn, with sequences playing a vital role.
\end{abstract}


\def\name{{\sc nemo}}

\section{Introduction}

How does the activity of individual neurons and synapses lead to higher-level cognitive functions? This is a central mystery in neuroscience which currently lacks an overarching theory. In \citet{papadimitriou2020brain} a mathematical model of the brain was proposed in an attempt at such a theory.  This neural model --- which we call \name\ in this paper --- entails brain areas, spiking neurons, random synapses, local inhibition, and plasticity (see the next subsection for a detailed description of \name).   The dynamical system defined by \name\ has certain emergent attractors corresponding to {\em assemblies of neurons;} recall that an assembly is a stable set of highly interconnected neurons in an area, representing through their (near) simultaneous excitation a real-world object, episode, or idea \citep{hebb1949organization, harris2003organization, buzsaki2019brain}.  There is a growing consensus in neuroscience that assemblies of neurons play an important role in the way the brain works \citep{buzsaki2010neural, huyck2013review, yuste2015neuron, eichenbaum2018barlow}.  It was established in \citet{papadimitriou2020brain} and subsequent research, through both mathematics and simulation, that certain elementary behaviors of assemblies arise in \name:  projection, association, merge, among others. Moreover, through \name\  one can implement certain reasonably complex cognitive phenomena, including learning to classify well-separated classes of stimuli \citep{dabagia2022assemblies}, and parsing natural language sentences \citep{mitropolsky2021biologically}.

Many of the brain's remarkable capabilities rely on working with {\em sequences} (of stimuli, words, places, etc), with the human brain's acumen for language being a particularly striking example. The capacity to memorize sequences is widely attested in cognitive neuroscience \citep{sugar2019episodic, bellmund2020sequence}. Experiments have documented the creation and activation of assemblies in sequence in the animal brain after training on tasks that involve sequential decisions \citep{dragoi2006temporal, pastalkova2008internally, dragoi2011preplay}. \citet{ikegaya2004synfire} observed that precisely-timed patterns of activation across large groups of neurons in the mouse neocortex are frequently repeated, suggesting memorization, and moreover that these patterns are combined into higher-order sequences (i.e., sequences of sequences). In the hippocampi of rats performing a sequential decision-making task, \citet{pastalkova2008internally} identified assembly sequences which predicted the decisions made. Across both navigational and memory tasks, the particular sequence which was exhibited depended closely on initial conditions and the structure of the task instance. In a further investigation by \citet{dragoi2011preplay}, sequences of neurons that were observed during a novel experience also occurred in spontaneous activity during rest before the experience was initiated again, a phenomenon known as ``preplay''. This finding suggests that the structure of sequence representations is largely determined by the intrinsic connectivity of the network.  

Arguably, it is through sequences of stimuli and their representations that brains deal with the all-important concept of {\em time}.  The question arises: Can \name\  capture this capability of the animal brain?   In past work, \name\  did not have to deal explicitly with sequences or time. In the English parser implemented in \cite{mitropolsky2021biologically}, the input sentence is presented sequentially, and the order of its words is not memorized by the device.  In subsequent work on parsing \cite{mitropolsky2022center}, the need to memorize subsequences of the input language became apparent in connection to the {\em center embedding} of sentences; however, no mechanism for this memorization was proposed.

In this paper, we demonstrate the emergent formation of sequences of assemblies in \name. When a brain area is stimulated by the same sequence of stimuli a handful of times, assemblies are reliably created, and the entire sequence will subsequently be recalled when only the beginning of the stimulus sequence is presented. Importantly, the underlying mechanism involves the capture of time precedence between sequences through the establishment, via plasticity, of high synaptic weights between stimuli representations, in the direction of time.  Moreover, we demonstrate that involving additional brain areas during presentation (essentially forming a ``scaffold'' of interconnected assemblies) makes memorization faster and more robust, and more so if this new area already contains another memorized sequence. This provides theoretical support to experience: It is easier to memorize a sequence of stimuli when each stimulus is mentally associated by the subject with an element of an already memorized sequence --- for example, a familiar tune, or the sequence of buildings next to the subject's home.

We use these ideas to further show that, in \name, assemblies can be configured to simulate an arbitrary {\em finite state machine} (FSM, or {\em finite state automaton}). Recall that FSMs are simple computational devices capable of recognizing and generating the class of sequential patterns known as {\em regular languages} \citep{sipser1996introduction}.  Moreover, we show that this configuration can be learned quickly by presenting sequences of stimuli corresponding to state transitions; this captures the brain's ability to learn {\em algorithms} involving sequences. The implementation and learning of FSMs relies crucially on one last feature of \name, namely {\em long-range interneurons (LRIs).} These are neural populations extrinsic to the brain areas of \name, which can be recruited by assemblies in adjacent areas, and whose function is to inhibit or disinhibit remote brain areas to achieve synchrony and control of the computation \citep{sik1995hippocampal, jinno2007neuronal, zhang2014long}.  There is evidence in the literature \cite{roux2015tasks} that LRIs are essential for the onset of $\gamma$ oscillations, often considered coterminous with brain computation.

One interesting byproduct of the mechanism for implementing and learning FSMs is a simple demonstration that \name\ is {\em Turing complete}. In other words, \name\  with LRIs constitutes a hardware language capable of implementing any computation, within the constraints imposed by the parameters of the model.  This is rather significant for a mathematical model that has the ambition to capture a large part of human cognition.  The original exposition of \name\  in \cite{papadimitriou2020brain} did contain an argument of Turing-completeness as well; however, that proof relies on a biologically implausible computer-like program, with loops, conditional statements, and variables corresponding to assemblies.  The Turing completeness argument in the present paper is carried out strictly within \name, and the required program is implemented with LRIs, yielding an entirely hardware-based general computer, consistent with neurobiological principles. 

\subsection{The Neural Model}
We next describe in detail \name, the mathematical model of the brain we consider here. The essential ingredients are a weighted, directed, random graph, along with simple mechanisms for determining which neurons fire ($k$-winners-take-all) and adjusting the weights (Hebbian plasticity and homeostasis).

\paragraph{Brain areas.}  The {\em brain} consists of a finite number of \emph{brain areas}.  A brain area is a set of $n$ excitatory neurons, connected internally by directed random edges (synapses), each present independently with probability $p$. Two brain areas $A$ and $B$ may be connected to each with edges in one direction (e.g., $A$ to $B$) --- or possibly in both directions --- by a {\em fiber}.  Fibers are directed bipartite random graphs; that is, if there is a fiber from $A$ to $B$, then for each pair of excitatory neurons $(i,j)$ where $i$ is in $A$ and $j$ is in $B$, there is a synapse from $i$ to $j$ with probability $p$, independently of all other possible synapses. All synapses have {\em weights}, assumed to be initially one.  An {\em input} (or {\em sensory}) area is a set of excitatory neurons, subsets of which are activated by external stimuli (e.g., the output of a sensory pathway, or the memory system) and excite brain areas it is connected to. (Fibers only go from sensory areas to other brain areas, not vice versa.)

\paragraph{The dynamical system.} We assume that computation proceeds in discrete time steps.  At each step, for each neuron its total synaptic input is determined by summing up the current weights of incoming connections from neighbors which fired on the previous time step. For each brain area, {\em only the $k$ neurons with the highest total input fire at each step} (with ties broken randomly).  This is an important part of the model, capturing local inhibition and the area's inhibitory/excitatory balance.  
Synapse weights, both within and between areas, are nonnegative and governed by Hebbian plasticity with parameter $\beta > 0$, so that each time $j$ fires immediately after $i$ fires, the weight of the synapse from $i$ to $j$ increases by a multiplicative factor of $1 + \beta$. In mathematical notation, at time $t$, let $x_A(t)$ be the set of neurons in area $A$, $W_A(t)$ be the recurrent weight matrix for area $A$, $W_{B, A}(t)$ the weight matrix from area $B$ to area $A$, and $I_A(t)$ the inhibitory signal for area $A$. Then we have
\begin{align*}
        x_A(t+1) &\gets I_A(t) \cdot k\text{-cap}\left(\sum_{B} W_{B, A}(t) x_B(t) \right)\\
        W_A(t+1) &\gets W_A(t) + \beta \, (x_A(t+1) \, x_A(t)^\top) \odot W_A(t)\\
        W_{B, A}(t+1) &\gets W_{B,A}(t) + \beta \, (x_A(t+1) \, x_B(t)^\top) \odot W_{B, A}(t) 
\end{align*}
where $\odot$ is element-wise multiplication, and the function $k\text{-cap}$ maps a vector to the indicator vector of its $k$ largest components (with ties broken randomly). 

\paragraph{Long-range interneurons.}
We assume that any brain area $A$ can be inhibited --- that is, no excitatory neuron in $A$ can fire --- by the activation of a designated population of inhibitory neurons, $I_A$. The activity of $I_A$ can in turn be suppressed by the firing of a group of {\em disinhibitory} neurons $D_A$, which receive input from other brain areas. We assume that these two types of interneurons integrate much more quickly than excitatory neurons --- as is generally accepted \citep{cruikshank2007synaptic, zhou1998ampa} --- and as a result these inhibition and disinhibition actions can be thought of happening instantaneously. 

\subsection{Related work}


The \name\ model is distinguished from other general theories of neuronal coordination by its bottom-up approach to the problem: It consists only of tractable abstractions of well-understood biological mechanisms, which can be shown to yield interesting behavior.
Among other such theories, the neuroidal model presented by \citet{valiant2000circuits, valiant2000neuroidal, feldman2009experience} is powerful but demands neurons and synapses capable of arbitrary state changes, while attempts to translate the success of deep learning via gradient descent into the brain \citep{lillicrap2016random, sacramento2017dendritic, guerguiev2017towards, sacramento2018dendritic, whittington2019theories, lillicrap2020backpropagation} rely on hypothetical feedback connections and exceedingly precise coordination.

Specifically in regard to memorizing and reproducing sequences, there are several other approaches which vary in biological fidelity and corresponding limitations. \citet{eliasmith2012large} is a biologically-plausible model of the entire brain which exhibits sequence memorization and prediction (along with other more complex cognitive tasks), but the mechanisms for doing so are complex and engineered, rather than emerging from simple dynamics. The model of \citet{cui2016continuous} is able to memorize and predict sequences in much the same fashion that we present here, using Hebbian plasticity and $k$-winners-take-all to abstract inhibition, but requires more sophisticated compartmental neurons and columnar organization. Lastly, the Tolman-Eichenbaum machine \citep{whittington2020tolman} is an abstraction of the hippocampus which can memorize sequences (along with more general structured sets of stimuli), but its encoding and mechanisms are simply assumed and it is trained with gradient descent. In comparison with these, \name is distinguished by (i) the simplicity and biological fidelity of its mechanisms and (ii) the fact that its interesting behavior is emergent, with minimal ``engineering''. 

A model very closely related to \name\ and called {\em the assembly calculus (AC)} was described in \citet{papadimitriou2020brain}, where it was shown that it is Turing complete.  The simulation of a Turing machine was accomplished using what amounts to a stored {\em program} of control commands, entailing variables, conditional statements, and loops, including commands needed for inhibiting and disinhibiting brain areas. A key contribution of this paper is that we eliminate such programs and use biologically plausible LRIs with similar consequences. As a result, our simulation is entirely self-contained, and driven only by the presentation of stimuli --- the higher-level connectivity of brain areas and interneurons can implement the general architecture of a Turing machine (i.e. a long tape of symbols and a tape head which keeps track of the current state and reads and writes symbols on the tape), while the weights of connections between neurons store the actual symbols on the tape and the state transitions of the tape head. 

\paragraph{Sequences of assemblies in neuroscience.} There is a large body of experimental evidence from neuroscience indicating that neural representations of temporally-structured patterns of stimuli often preserve that structure. The assemblies hypothesis goes a step further: these neural representations ought not to arise only in response to the sequence of stimuli, but should also be generated internally by the network, for example by playing out the entire sequence of neural responses in response to only a fragment of the entire pattern of stimuli. Here we recount certain inspiring experimental results which support this view. In \citet{dragoi2006temporal}, neurons in the hippocampi of rats running a linear maze were recorded, and the pairwise temporal correlations of their spikes were computed. These correlations were significantly larger than predicted by the model where neurons simply fire in response to where the animal currently is, suggesting that neurons which fire in response to one location of the maze become linked with those in nearby locations. \citet{pastalkova2008internally} trained rats to alternate between two paths of a maze, with a delay period between each run, while again recording from the hippocampus throughout. They found that the activity during the delay period was a strong predictor of which path the rat would take, even on error runs (where the animal failed to alternate between paths correctly). As the environmental context was identical during these delay periods, this points to internal generation of these neural responses. Finally, \citet{dragoi2011preplay} recorded from the hippocampi of mice during periods of sleep and exploration of an unfamiliar maze, and found that the sequences of firing which occurred during exploration matched those that occurred during prior sleep periods. These sequences were active before the animal even saw the novel maze, which suggests that particular temporal patterns of neural activity arise spontaneously, and these patterns are later recruited to represent environmental stimuli. Although the sequences of firings which arise in our model need not occur spontaneously before the first presentation of the stimulus sequence, they do rely crucially on internal connectivity, and by repeated presentation of the stimulus sequence become linked together to reinforce the replay of the sequence.

\section{Computation with sequences of assemblies}
\subsection{Sequence memorization}

We begin with {\em sequence projection,} where a sequence of assemblies from one area is projected to another area. 
The most natural way to project a sequence from one area to another is to simply activate, in the first area, the assemblies for each element of the sequence, one after the other in the given order. Assuming that there is a fiber connecting the first area to the second, and the second area is disinhibited, one would hope that this would result in the creation of a set of corresponding assemblies in the target area, so that activating any newly created assembly results in the sequential activation of the rest of the sequence of new assemblies. This is the essence of our first finding, stated below with quantitative bounds on plasticity and the other parameters.

\begin{figure}
    \centering
    \includegraphics[width=0.3\linewidth]{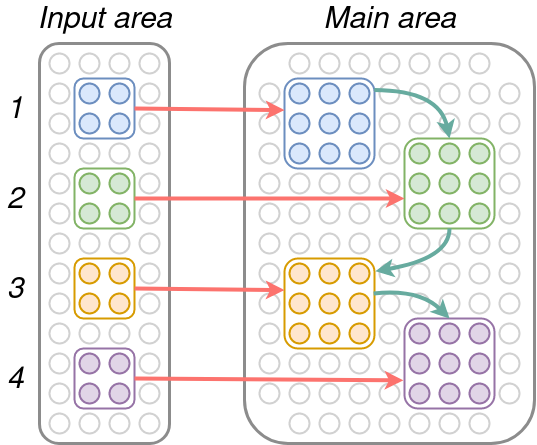}
    \hspace{0.1\linewidth}
    \includegraphics[width=0.52\linewidth]{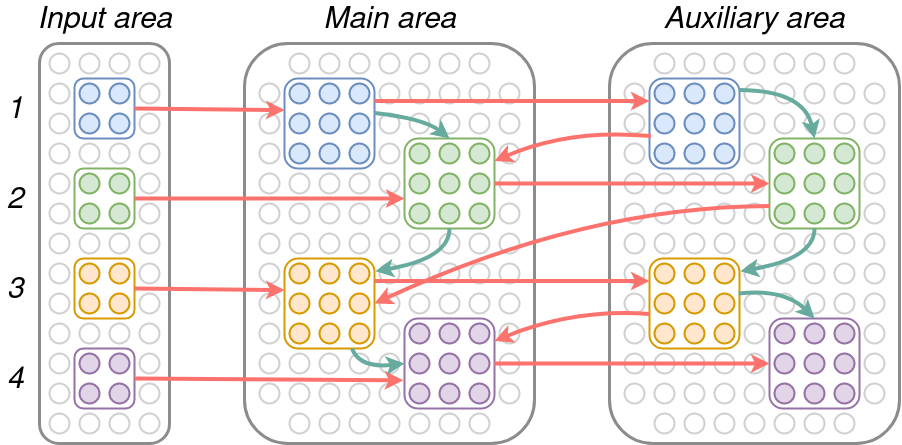}
    \caption{Left: Simple copy of a sequence of assemblies. When a sequence is played a few times in one area, a corresponding sequence of assemblies is formed in an adjacent area, and will subsequently be recalled when only the beginning of the sequence is presented in the input area. Right: When the area with the copied sequence is also allowed to send and receive input to/from an auxiliary area, a ``scaffolded'' sequence of assemblies is formed across the two areas, leading to faster, more reliable memorization.}
    \label{fig:seq_schematic}
\end{figure}

\begin{theorem}[\textbf{Simple sequence copy}] \label{thm:simple}
Suppose that area $A$ receives input from area $S$. Let $S_1,\ldots$, $S_L$ be a sequence of subsets of $k$ neurons from $S$, with $L \le n / 2k$, and $|S_{\sigma} \cap S_{\sigma'}| \le \Delta$ for all $\sigma \neq \sigma' \in \{1,2,\ldots,L\}$, and $\Delta$ a positive integer. Suppose this sequence is presented $T$ times (with each presentation beginning from rest) with area $A$ disinhibited, forming a sequence of caps $A_1(t), \ldots, A_L(t)$ in area $A$ on round $t \le T$. After each round, homeostasis is applied, so that each neuron's incoming weights sum to $1$. Then for 
\[
kp \ge 3 \ln n, \qquad \Delta \le \frac{k}{(2\ln n)^2}, \qquad \beta \le \frac{ \ln \frac{n}{2kL}}{2 (\max \{\Delta p, 6 \ln n\})^2 }, \qquad 
T \ge \frac{1}{\beta \ln (n/k)}
\]
when any $S_i$ fires once in the input area, and $A$ is allowed to fire $L-i+1$ times, the resulting sequence of caps $\widehat A_i, \ldots, \widehat A_L$, satisfies
\[\frac{\E[|\widehat A_j \cap A_j(1)|]}{k} \ge 1 - \left(\frac{k}{n}\right)^{2 \beta T}\]
for all $j \ge i$. In particular, for 
\[T  \ge \frac{1}{2\beta}\sqrt{\frac{\ln(nL)}{\ln (n/k)}}\]
w.h.p. we get perfect recall of the entire sequence after projection.
\end{theorem}

In other words, following $T$ rounds of rehearsal, subsequent presentation of any input assembly in the sequence results in the activation of the corresponding assembly and the rest of the sequence in the target area. We note that the number of presentations needed grows inversely with the plasticity, and crucially, the plasticity cannot be too high.

\paragraph{Memorization with a scaffold.} A well-known phenomenon in cognitive science is that memorization is easier by creating associations such as mnemonics. For sequences, learning one sequence by associating with another sequence, element by element (e.g., learning the alphabet to a tune), helps with retention and recall. We consider a very simple form of this: when projecting a sequence, we create two copies of the sequence that ``scaffold" each other. Remarkably, this leads to provably better recall of the sequence with about half as many training rounds as simple sequence projection. 

\begin{theorem}[\textbf{Scaffold sequence copy}] \label{thm:scaffold}
Suppose that areas $A$ and $B$ are connected to each other, and only $A$ receives input from area $S$. Consider a sequence $S_1, \ldots, S_L$ of sets of $k$ neurons in area $S$, for $L \le n / 2k$, which satisfy $|S_{\sigma} \cap S_{\sigma'}| \le \Delta$ for all $\sigma \neq \sigma'$. Suppose this sequence is presented $T$ times (with each beginning from rest) with areas $A$ and $B$ disinhibited, to form a sequence of caps $A_1(t), \ldots, A_L(t)$ within $A$, and $B_1(t), \ldots, B_L(t)$ within $B$ on round $t$. After each round, homeostasis is applied, so that each neuron's incoming weights sum to $1$. Then for 
\[ kp \ge 3 \ln n, \qquad \Delta \le \frac{k}{(2\ln n)^2}, \qquad \beta \le \frac{\ln \frac{n}{2kL}}{2 (\max \{\Delta p, 6 \ln n\})^2 }, \qquad T \ge \frac{1}{\beta \ln (n/k))}\]
when any $S_i$ fires once (for $1 \le i \le L$), and $A$ and $B$ are allowed to fire $L-i+1$ times to form a sequence of caps $\widehat A_i, \ldots, \widehat A_L$ and $\widehat B_i, \ldots, \widehat B_L$, we will have
\[\frac{\E[|\widehat A_j \cap A_j(1)|]}{k}, \frac{\E[|\widehat B_j \cap B_j(1)|]}{k} \ge 1 - \left(\frac{k}{n}\right)^{4 \beta T}\]
for any $j \ge i$.
In particular, for 
\[T \ge \frac{1}{4 \beta}\sqrt{\frac{\ln(nL)}{\ln (n/k)}}\]
w.h.p. we get perfect recall of the entire sequence in both areas $A$ and $B$.
\end{theorem}

We note that the bound on the number of training rounds is a factor of two smaller here compared to simple sequence projection (Figure \ref{fig:seq_compare}).

\subsection{(Dis)inhibition with interneurons}

Interneurons serve to inhibit and disinhibit areas. The lemma below shows that by connecting $D_A$ to $B$ and $D_B$ to $A$, we get alternate firing of the two areas $A$ and $B$. More precisely, the firing of a cap in area $B$ causes $D_A$ to fire, which causes $I_A$ to cease firing, which allows a cap in $A$ to fire in response to input from $S$ and $B$ (see Figure \ref{fig:alt_demo}).

\begin{figure}[h!]
    \centering
    \begin{subfigure}[h]{0.32\textwidth}
    \begin{tabular}{c|c}
        Time step & Firing\\
        \hline
        1 & $S_1, D_A, I_B$\\
        2 & $S_2, A_1, D_B, I_A$\\
        3 & $S_3, B_1, D_A, I_B$\\
        4 & $S_4, A_2, D_B, I_A$\\
        5 & $S_5, B_2, D_A, I_B$\\
        \vdots & \vdots\\
    \end{tabular}
    \end{subfigure}
    \hspace{0.1\textwidth}
    \begin{subfigure}[h]{0.32\textwidth}
    \includegraphics[width=\textwidth]{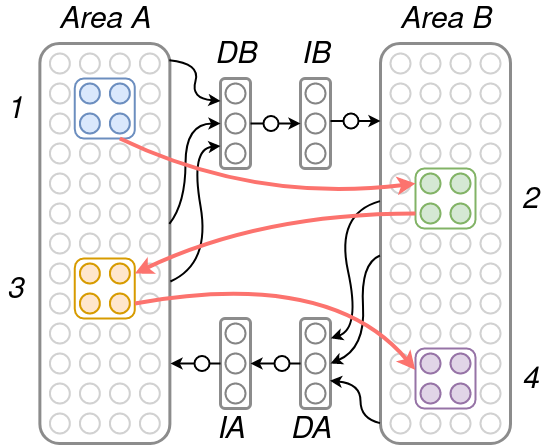}
    \end{subfigure}
    \caption{An example sequence of firings with interneurons. The table on the left shows the sets of neurons firing simultaneously on each round; on the right is the architecture of the network.}
    \label{fig:alt_demo}
\end{figure}

\begin{lemma}[\textbf{Alternation}] \label{lemma:inter}
Let $D_A$ receive input from $B$ and $D_B$ receive input from $A$, with $A$ and $B$ connected to each other and both receiving input from $S$.
For any sequence of inputs $S_1, S_2, \ldots$ in the input area, where $D_A$ and $I_B$ are initially firing, the resulting sequence of activations $S_1', S_2'\ldots$ satisfies $S_\sigma' \subseteq A$ for $\sigma$ odd and $S_\sigma' \subseteq B$ for $\sigma$ even. 
\end{lemma}

\begin{proof}
By induction on $\sigma$. For the base case $\sigma=1$, as $I_A$ is not firing initially and $I_B$ is, only neurons in $A$ will fire. So, $S_1' \subseteq A$. More generally assume the claim holds for all $\sigma' < \sigma$. If $\sigma$ is even, then $S_{\sigma-1}' \subseteq A$ and $S_{\sigma-2}' \subseteq B$ so $D_B$ and $I_A$ fired on the previous round. Hence, $B$ is disinhibited and $A$ is inhibited, so $S_\sigma' \subseteq B$. If $\sigma$ is odd, then $S_{\sigma - 1}' \subseteq B$ and $S_{\sigma-2}' \subseteq A$, so $D_A$ and $I_B$ fired on the previous round. Hence, $A$ is disinhibited and $A$ is inhibited, so $S_\sigma' \subseteq A$.
\end{proof}

We will make extensive use of this property in the next section.

\subsection{Learning finite state machines}

Here we demonstrate that \name\  is powerful enough to simulate an arbitrary finite state machine (FSM). In fact, an FSM is learned --- that is, memorized --- simply by presenting all valid transitions between states in the FSM.

\paragraph{Finite state machines.} A finite state machine (FSM) is a tuple $F = (Q, \Sigma, q_0, q_A, q_R, \delta)$, where $Q$ is a set of states, $\Sigma$ is a finite input alphabet, $q_0 \in Q$ is the initial state, and $\delta \colon Q \times \Sigma \to Q$ is the transition function, which maps a (current state, input symbol) pair to a new state (see Figure \ref{fig:fsm} for an example). Given an input string $\sigma_1 \sigma_2 \ldots$ and starting from the input state $q_0$, at time $t \ge 1$ the FSM goes from state $q_{t-1}$ to $q_t = \delta(q_{t-1}, \sigma_t)$.

The states $q_A, q_R \in Q$ are special terminal states, along with a special terminal character $\square \in \Sigma$ which only appears at the end of the input string. Every state $q \not\in \{q_A, q_R\}$ satisfies $\delta(q, \square) \in \{q_A, q_R\}$. When the machine reaches state $q_A$ or $q_R$, we say that it accepts or rejects the string, respectively.

\begin{figure}
    \centering
    \includegraphics[width=0.35\textwidth]{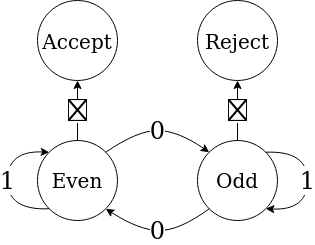}
    \hspace{0.1\textwidth}
    \includegraphics[width=0.4\textwidth]{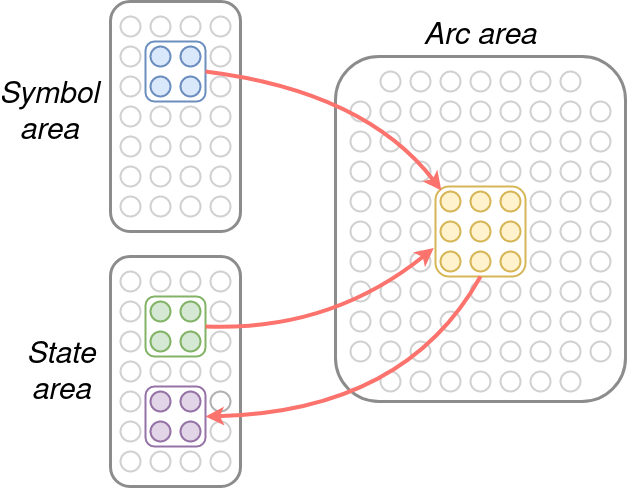}
    \caption{On the left is an example of a finite state machine (FSM), which consists of a finite number of states, joined by transitions (arrows). The machine changes states based on input symbols. This FSM accepts binary strings which contain an even number of zeros. On the right is the network architecture used here to simulate finite state machines. There is an assembly for each symbol, state, and transition; each pair of state and symbol assemblies projects to the associated arc assembly, which in turn projects back to the assembly corresponding to the state the FSM would switch to after seeing that state/symbol combination..}
    \label{fig:fsm}
\end{figure}

\begin{theorem} \label{thm:fsm}
Let $F = (Q, \Sigma, q_0, q_A, q_R, \delta)$ be an FSM. Consider a network consisting of an input area $I$ and brain areas $S$ and $A$, where $I, S$ both have connections to and from area $A$, with interneurons $D_S, D_A$, where $D_S$ disinhibits $S$ upon input from $A$ and $D_A$ disinhibits $A$ upon input from $S$. Suppose that for each $q \in Q, \sigma \in \Sigma$, there exist designated sets of $k$ neurons $S_q \subseteq S$, $I_\sigma \subseteq I$, where $|S_q \cap S_r|, |I_\sigma \cap I_\rho| \le \Delta = o(k)$ for all $q \neq r, \sigma \neq \rho$. Then if 
\[n \ge |Q|^2|\Sigma|^2, \qquad kp \ge 6 \ln n / k, \qquad \beta \le \frac{ \ln \frac{n}{2kL}}{2 (\max \{\Delta p, 6 \ln n\})^2 }, \qquad T \ge \frac{12}{\beta} \sqrt{\frac{\ln n}{kp}} \] w.h.p. this network can simulate $F$ on any input string $\sigma_1 \sigma_2 \cdots \sigma_L$ in the following sense: If $S_{q_0}$ is made to fire on round $1$, and $I_{\sigma_i}$ is made to fire on round $2i - 1$, $1\le i \le L$, then after $2L+2$ rounds, $S_{q_A}$ will fire if $F$ accepts the string and $S_{q_R}$ will fire if $F$ rejects the string.
\end{theorem}

Note that by Lemma \ref{lemma:inter} if $S$ fires on some round, $A$ will be permitted to fire on the next, and vice versa. The crux of the simulation is to enable the configuration of synaptic weights of the network so that together, $S_q$ and $I_\sigma$ project to an assembly $A_{q, \sigma} \subseteq A$, and in turn $A_{q, \sigma}$ projects to $S_{\delta(q, \sigma)}$ (see Figure \ref{fig:fsm} for a schematic). As part of the proof, we will show that firing the sequence $\{S_q, I_\sigma\}, A_{q, \sigma}, S_{\delta(q, \sigma)}$ at least $7 / \beta$ times will suffice to arrange this, so the FSM simulation can actually be configured by simply observing state transitions. With this in place, if $S_q$ and $I_\sigma$ fire together, two rounds later $S_{\delta(q, \sigma)}$ will fire, simulating a single transition of the FSM. So, with $S_{q_0}$ firing initially and $I_{\sigma_i}$ firing after $2L$ rounds, $S_{q_A}$ (resp. $S_{q_R}$) will fire if the FSM accepts (resp. rejects). Hence, the behavior of the FSM on any input string $\sigma_1 \ldots \sigma_L$ can be simulated in the projected FSM by setting $S_{q_0}$ to fire initially in area $S$ and presenting $I_{\sigma_1}, I_{\sigma_2}, \ldots$ every other time step in area $I$.  In Figure \ref{fig:fsm_emp}, we show assemblies simulating the FSM from Fig.~\ref{fig:fsm}.

\paragraph{An illustrative FSM.} In Figure \ref{fig:fsm_example}, we show assemblies simulating a FSM which recognizes strings which consist of the base $10$ representation of a number divisible by $3$. With $3$ non-terminal states and an alphabet of $10$ symbols, it operates by tracking the cumulative sum of the digits modulo $3$, and accepts if this modulus is $0$ at the end of the string (rejecting otherwise).


\begin{figure}[t]
    \centering
    \includegraphics[width=0.6\textwidth]{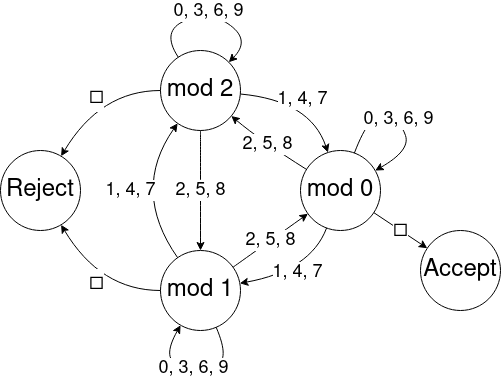}\\
    
    \includegraphics[width=\textwidth]{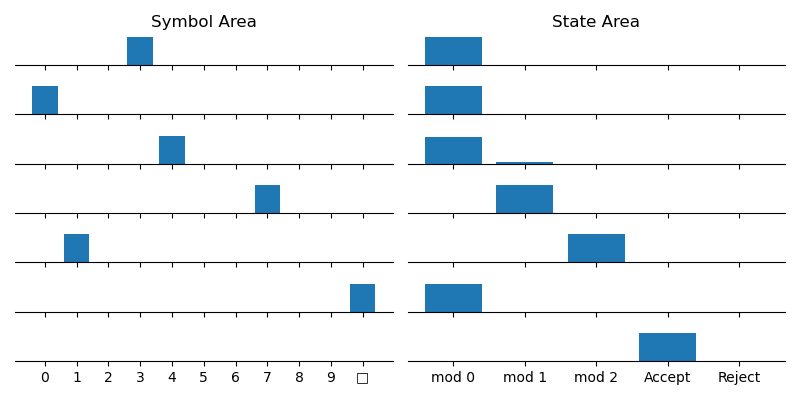}
    \caption{Above is the schematic of a FSM which accepts numbers in base $10$ which are multiples of $3$. The non-terminal states of the machine correspond to the remainder of the sum of digits which have been seen; the FSM accepts if the remainder at the end of the string is 0, and rejects otherwise. This FSM has a total of 5 states, 11 symbols, and 33 transitions. In the lower plot, we display the fraction of overlap between the responses of the trained model and the assemblies representing states and symbols, while simulating this FSM on the test string \texttt{30471}. Here, $n=5000, k=70, p=0.4, \beta=0.1$, and the model was trained with 15 presentations of each transition.}
    \label{fig:fsm_example}
\end{figure}


\begin{remark} A finite-state transducer (a finite state machine that produces an output symbol with each transition) with output function $\theta \colon Q \times \Sigma \to \Gamma$ can be simulated using a third area $B$, which contains designated sets of $k$ neurons $B_\gamma$ for each $\gamma \in \Gamma$. To accomplish this, one simply fires $B_{\theta(q, \sigma)}$ at the same time as $S_{\delta(q, \sigma)}$ during training. Then, w.h.p., $B_{\theta(q, \sigma)}$ will fire two steps after $A_q$ and $I_\sigma$ under the same conditions on $T$.
\end{remark}

\subsection{Turing Completeness}
A Turing Machine (TM) 
is an FSM together with a tape (read-write memory). Concretely, a single-tape TM is a tuple $M = (Q, \Sigma, \{L, R\}, q_0, q_A, q_R, \delta)$ where $Q$ is a set of states, $\Sigma$ is a set of tape symbols, $q_0$ is an initial state, and $\delta \colon Q \times \Sigma \to Q \times \Sigma \times \{L, R\}$ is a transition function. The TM has access to a tape of symbols, which initially has written on it the input to the machine, with an infinite number of blank space symbols ($\sqcup$) extending to the left and right of the input. The tape head of the TM indicates the current symbol, and begins at the first symbol of the input. The transition function of a Turing machine $\delta \colon Q \times \Sigma \to Q \times \Sigma \times \{L, R\}$ maps the current state and symbol indicated by the tape head to a new state, a new symbol to write, and a direction to move the tape head. The operation of the machine is as follows: At each time step, when it is in state $q$, it reads a symbol $\sigma$ from its tape. Denote the output of the transition function as $\delta(q, \sigma) = (r, \rho, d)$. It then (i) changes state to $r$, (ii) replaces $\sigma$ with $\rho$ on the current tape square, and (iii) moves to the square immediately left of the current one if $d = L$ and right if $\Delta = R$. For simplicity in the simulation, we require that $\rho = \sigma$ if $d = L$, which is easily seen to maintain generality. 
Notably, the head of a Turing machine may be viewed as a finite state transducer which outputs symbols from the alphabet $\Sigma \times \{L, R\}$, writing symbols back to its (unbounded) input tape. 

To simulate an arbitrary transition function, we augment the FSM network consisting of areas $I, S, A$ with areas $D$ and $M$, which are where movement commands and symbols to be written will be output (respectively). Areas $I$ and $M$ contain assemblies $I_\sigma$ and $M_\sigma$, respectively, for each $\sigma \in \Sigma$, $S$ contains an assembly $S_q$ for each $q \in Q$, and $D$ contains assemblies $D_L$ and $D_R$. Now, as in Theorem \ref{thm:fsm}, to implement a transition $\delta(q, \sigma) = (r, \rho, d)$, assemblies $I_\sigma$ and $S_q$ project to an assembly $A_{q, \sigma} \subseteq A$. In turn, $A_{q, \sigma}$ projects to $S_r, M_\rho$, and $D_d$.

What remains is to simulate the tape, namely maintain a sequence of symbols and a pointer location, and update both according to the output of the simulated transition function.
We will show how to simulate a tape with assemblies, and thus (together with the FSM simulation) a Turing machine in its entirety, momentarily. The idea of the tape simulation is to maintain an assembly for each nonempty tape square, which projects to the assembly corresponding to the appropriate symbol. Each tape assembly is linked with those representing neighboring tape squares, and distributed across several brain areas configured so that when the FSM simulation issues a movement command, the assembly corresponding to the next tape square in the direction of movement will fire. To overwrite the current tape square, a new assembly is created and its connections with its neighbor and the new symbol are strengthened. 

We remark that if the model instead has access to an external ``tape", we can achieve general computation with essentially just the FSM.  For the environment to implement a tape, it suffices for it to contain a large number of distinguishable spaces, which can each store a representation of a symbol from $\Sigma$ (imagine a stack of index cards and a writing implement). There is a pointer which indicates the current space, which is driven by the activity of the network such that it moves left or right when $D_L$ or $D_R$ fire in area $D$. If symbol $\sigma$ is contained in the space indicated by the pointer, it causes $I_\sigma$ to fire in area $I$; and if $M_\sigma$ fires, symbol $\sigma$ is written to the space.
 The system consisting of the model and its environment will implement the action of the Turing machine with transition function $\delta$ on whatever is initially written on the external tape.

\subsubsection*{Tape simulation}
To simulate the tape using assemblies, we split the tape into two halves and simulate each half independently. 
Each half-tape supports two operations: One operation adds a symbol to the beginning of the tape, while the other removes the current symbol so that the succeeding symbol is at the beginning. We implement this tape with assemblies as follows. 

Intuitively, the tape is represented by a sequence of assemblies which cycle between the three areas, with the current position of the beginning of the tape represented by the currently firing assembly (see Figure \ref{fig:tape_sim}). To remove the current symbol, the next area in the cycle is disinhibited, which causes the next assembly in the chain to fire (and the previous assembly is effectively forgotten); to add a new symbol, a new assembly is created in the preceding area, linked to the current assembly. The symbol at each position of the tape is stored by strengthened connections between the tape assemblies and symbol assemblies in a designated symbol area, which will fire when the associated tape assemblies do. Precisely, the simulation consists of maintaining a set of assemblies $A_1(t), \ldots, A_{L_t}(t)$ for $1 \le t \le T$ with the following properties:
\begin{enumerate}[label=(\roman*)]
    \item If the $t$'th operation adds $\sigma$ to the beginning of the tape, with $A_1(t) \subseteq H_i$, we will have $A_{1}(t+1) \subseteq H_{i-1}$, and $A{j+1}(t+1) = A_j(t-1)$ for all $1 \le j \le L_t$.
    \item If the $t$'th operation removes $\sigma$ from the beginning of the tape, we will have $A_i(t+1) = A_{i+1}(t)$ for all $1 \le i \le L_t-1$.
\end{enumerate}

A ``delete'' operation when assembly $A_1(t)$ in area $H_i$ is firing will simply cause $H_{i+1}$ to fire, and in turn strengthened connections from $A_1(t)$ to $A_2(t)$ will cause $A_2(t)$ to begin firing. An ``add'' operation is more complicated; with $H_{i-1}$ disinhibited, input from outside of the tape areas will drive creation of a new assembly $A_1(t+1)$, which is linked to $A_1(t) = A_2(t+1)$ and $S_\sigma$ (where $\sigma$ is the symbol to be added to the tape) by their simultaneous firing. Thus, we require an lower bound on the plasticity, so that $A_1(t+1)$ can be formed quickly, as well as an upper bound on the plasticity, to ensure that existing assemblies in area $H_{i-1}$ will not overlap $A_1(t+1)$ by too much. Three areas are required so that $A_1(t+1)$ becomes linked only in the forward direction to $A_1(t)$, and not in the backwards direction (i.e. in the future when $A_1(t+1)$ fires it will trigger $A_1(t) = A_2(t+1)$ to fire, but not vice versa).

\begin{figure}
    \centering
    \includegraphics[width=0.9\textwidth]{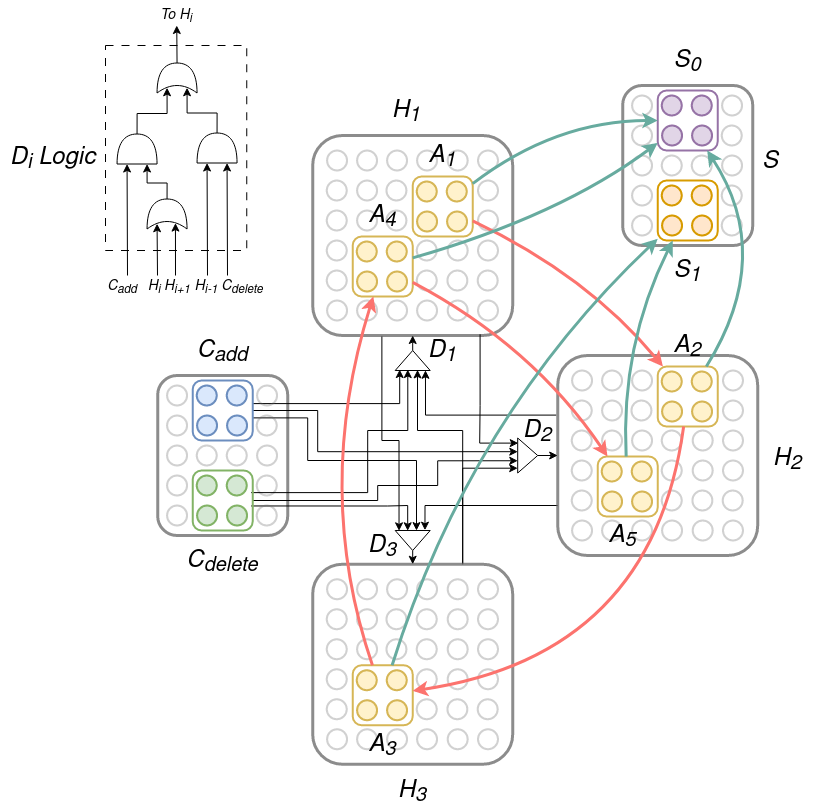}
    \caption{Simulating a restricted tape with assemblies. Assemblies $A_1, \ldots, A_5$ each represent a space on the tape, and are placed to cycle between areas $H_1, H_2, H_3$ (red arrows). $C_\text{add}$ and $C_{\text{delete}}$ are externally-activated control assemblies, which signal operations to areas $H_1, H_2, H_3$ and via interneuron populations $D_1, D_2, D_3$ determine the sequence in which they are disinhibited. On the upper left we provide a schematic of the logical function implemented by $D_i$, which carries out the inhibition and disinhibition of $H_i$ to simulate the operations signalled by the control assemblies. Black arrows in the larger diagram correspond to inputs to (i.e. synapses of) the interneurons, $D_1, D_2, D_3$. To represent symbol $\sigma$ being written at space $i$, assembly $A_i$ is linked to assembly $S_\sigma$ (teal connections) so that the firing of $A_i$ will cause $S_\sigma$ to fire. At this point in the simulation, the simulated tape has the string \texttt{00101} written on it.}
    \label{fig:tape_sim}
\end{figure}

\begin{lemma}\label{lemma:stack}[Tape Simulation]
Consider three brain areas $H_1, H_2, H_3$, each with $n$ neurons, where $H_i$ gives input to $H_{i+1}$ (addition modulo $3$) and all give input to an area $S$. The firing of $H_i$ is governed by interneurons $D_i$, which receive input from a control area $C$ containing assemblies $C_{\text{add}}$ and $C_{\text{delete}}$. Suppose the behavior of $D_i$ is as follows: $D_i$ permits $H_i$ to fire $T$ times when either $C_{\text{add}}$ fires together with $H_{i}$ or $H_{i+1}$, OR $C_{\text{delete}}$ and $H_{i-1}$ fire together. Moreover, there are two areas $E_1, E_2$, connected via interneurons which disinhibit $E_1$ for $T$ rounds after $E_2$ fires $T$ times, and disinhibit $E_2$ for $T$ rounds after $E_1$ fires. 
Suppose that $S$ initially contains assemblies $S_\sigma$ for each $\sigma \in \Sigma$, with $|S_\sigma \cap S_\rho| \le 6 \ln n$ for all $\sigma \neq \rho$. Let $E$ be a set of neurons containing a set of $k$-caps $Z_1, \ldots, Z_L$, where $|Z_t \cap Z_s| \le 6 \ln n$ for all $t \neq s$, and $Z_t$ fires on rounds $(t-1)T + 1$ through $tT$. Assume that
\[\sqrt{\frac{\ln n}{kp}} \le \beta \le \frac{\ln \frac{n}{2kK}}{72 \ln^2 n } \]
and $kp \ge 3 \ln n$.
Suppose there is a sequence of ``add'' and ``delete'' operations in the following sense: If the $t$th operation adds symbol $\sigma$, then $C_{\text{add}}$ and $S_\sigma$ are made to fire by external control on rounds $(t-1)T + 1$ through $tT$, while if the $t$th operation deletes the current symbol, then $C_{\text{delete}}$ fires on rounds $(t-1)T + 1$ through $tT$.  
For a sequence of length $L \le \frac{n}{2k}$, NEMO can simulate this sequence of operations: WHP, if $\sigma$ is the current symbol on round $t$, then $S_\sigma$ will fire on round $tT+1$. 
\end{lemma}

Now, by combining simulations of each tape half, we can simulate the entire tape. The symbol under the tape head is represented by the symbol at the beginning of the right half of the tape. A rightward movement of the TM is simulated by removing the first symbol from the right half of the tape, and adding the (potentially changed) symbol to the left half of the tape. A leftward movement is simulated by removing the first symbol from the left half and adding it to the right half (recall that the symbol does not change in a leftward movement). As a consequence, \name\ is Turing complete.

The simulation consists of six areas $H_i^L, H_i^R, i=1,2,3$, three each for the ``left'' and ``right'' halves of the tape, a control area $D$ for both halves jointly, and the three areas $I, S, A$ for the FSM, which yields a total of 10 areas ($2 \times 3 + 1 + 3$). The outline of its operation is as follows: The FSM simulation proceeds as in Theorem 4, modified so that if $\delta(q, \sigma) = (p, \rho, d)$, the assembly $A_{q, \sigma}$ causes $S_p$, $I_\rho$, and $D_d$ to fire, where $d \in \{L, R\}$. The tape halves are simulated as in Lemma \ref{lemma:stack}, where the firing of $D_L$ signals a ``delete'' operation for the left tape and an ``add'' operation for the right tape (and vice versa). The new assembly on the right half becomes linked to $I_\rho$ by their concurrent firing, which effects a leftward movement of the tape head (and similarly for $D_R$).

\begin{theorem} \label{thm:turing}[TM Simulation]
    With plasticity
    \[\sqrt{\frac{\ln n}{kp}} \le \beta \le \frac{\ln \frac{n}{2kT}}{72 \log^2 n } \]
    and $kp \ge 3\ln n$, w.h.p. \name\ can simulate a Turing machine $M = (Q, \Sigma, \{L, R\}, q_0, q_A, q_R, \delta)$ which uses $T$ time in time $O(T)$ using ten brain areas, each of size $n \ge 2\max\{T^2, |Q|^2|\Sigma|^2\}$. 
\end{theorem}

\section{Experiments}
To support our theoretical results, we have simulated the sequence and FSM models extensively. We detail the results and conclusions of these various experiments here.

\begin{figure}[h]
    \centering
    \includegraphics[width=\textwidth]{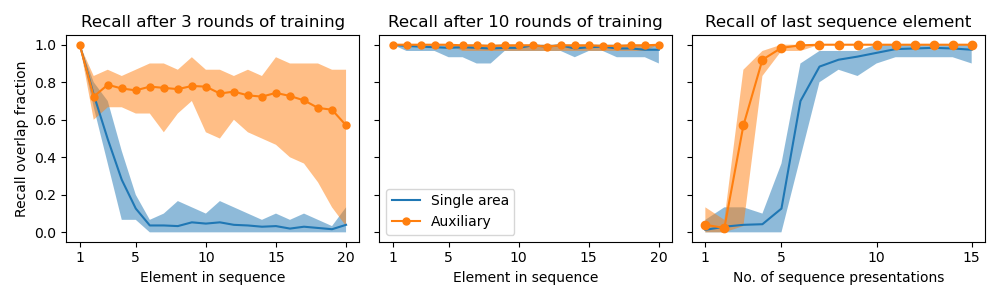}
    \caption{Simple vs scaffold sequence memorization:  we examine recall of the sequence in response to presentation of only the first item after 3 (left) and 10 (center) presentations of the sequence, by measuring the fraction of neurons in the assembly corresponding to a given element of the sequence which fire during testing. On the right, we demonstrate how the recall of the last element of the sequence (again, after only the first is presented) over the course of training. Here, $n=1000, k=30, p=0.2, \beta=0.1$, and the sequence is length $20$. Dark center line is the mean over 10 trials, while shaded area is the range.}
    \label{fig:seq_compare}
\end{figure}

\paragraph{Number of presentations versus recall.} In Figure \ref{fig:seq_compare} we show how recall of a sequence improves over the course of repeated presentations, for both the single-area and scaffolded models. Significantly, the faster rate of memorization observed for the scaffolded model supports the bound given in Theorem \ref{thm:scaffold}. Here, we estimate the assembly corresponding to a given element in the sequence as the set of neurons which fire on the last round of training, and measure recall as the fraction of neurons in that assembly which fire at the appropriate time when only the first element of the sequence is presented. 

\begin{figure}[h]
    \centering
    \includegraphics[width=0.4\textwidth]{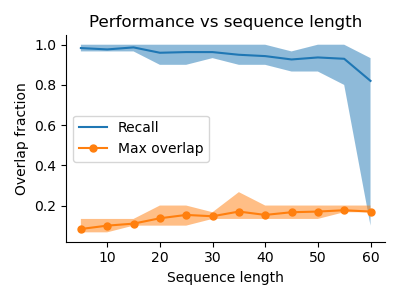}
    \caption{For each choice of sequence length, we train the model via repeated presentations of a sequence of the given length. In blue we plot the recall of the last element in the sequence, and in orange we plot the maximum overlap of assemblies corresponding to distinct sequence elements. Here, $n=1000, k=30, p=0.2, \beta=0.1$, and each sequence is trained via 10 presentations. Dark center line is the mean over 10 trials, while shaded area is the range.}
    \label{fig:seq_len}
\end{figure}

\paragraph{Sequence memorization capacity.} In Figure \ref{fig:seq_len}, we increase the length of the sequence to be memorized and measure both the recall of the sequence and the maximum overlap between assemblies corresponding to different sequence elements. This disambiguates two failure modes of sequence learning, where either different neurons fire during testing versus training or the assemblies representing different sequence elements become indistinguishable. Notably, the capacity of a given brain area seems to greatly exceed the bound we give in Theorem \ref{thm:simple}, and indeed even surpasses the quantity $n/k$ which is the largest number of disjoint $k$-caps which can be formed in a given area. This suggests that sequence learning in \name\ is significantly more robust than our analysis implies.

\paragraph{Effect of parameters on sequence memorization.} We examine the effect of area size (Figure \ref{fig:seqvary}, left) and edge density (Figure \ref{fig:seqvary}, right) on the success of sequence memorization. In each case, we vary the relevant parameter while training the model to memorize a length $25$ sequence, and measure the fraction of neurons in the last assembly of the sequence which fire at the appropriate time when only the first element of the sequence is presented (\emph{recall}), and the largest fraction of overlap between any two assemblies corresponding to distinct elements (\emph{max overlap}). High recall indicates that the neural responses to the sequence of stimuli have been memorized, while low max overlap indicates that the assemblies corresponding to any pair of distinct sequence elements are distinguishable. Our analytical results (where the probability of success increases with $n$, and we require the product $kp$ to be sufficiently large) predict that increasing either parameters should improve performance (i.e. increase recall and decrease overlap) which is indeed observed experimentally. Notably, reasonably good performance is attained for substantially larger ranges of $n$ and $p$ than we require in the theorems.

\begin{figure}[h!]
    \centering
    \includegraphics[width=0.45\linewidth]{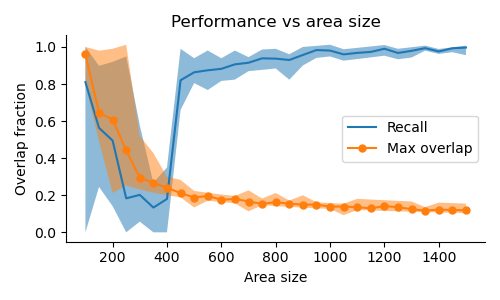}
    \hspace{0.0\linewidth}
    \includegraphics[width=0.45\linewidth]{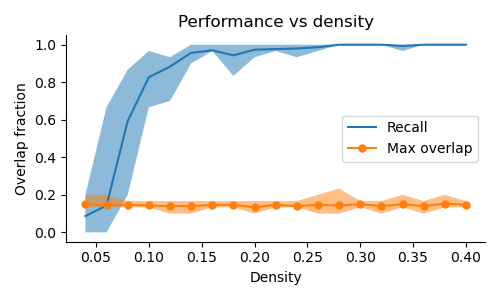}
    \caption{We measure performance of sequence memorization as two key parameters of the model (area size and density of connections) are varied. On the left, we vary $n$ from $100$ to $1500$ while $k = \sqrt{n}$, $p=0.2$, $\beta = 0.1$. On the right, we vary $p$ from $0.04$ to $0.4$ while $n = 1000, k=30, \beta = 0.1$. In both cases the sequence to be memorized is length $25$ and training consists of 10 presentations. Dark center line is the average over 10 trials, while shaded area is the range.}
    \label{fig:seqvary}
\end{figure}

\paragraph{FSM memorization.} In Figure \ref{fig:fsm_emp} we demonstrate how the performance of each transition of the FSM improves over the course of training, and how performance degrades as the size of the FSM (i.e. total number of transitions) increases, while all parameters of the model are fixed. We measure performance by the recall averaged over all transitions, where here the recall is the fraction of neurons in the appropriate next state assembly which fire two rounds after the assemblies corresponding to a given state/symbol pair are made to fire.

\begin{figure}
    \centering
    \begin{subfigure}[h]{0.4\textwidth}
        \includegraphics[width=\textwidth]{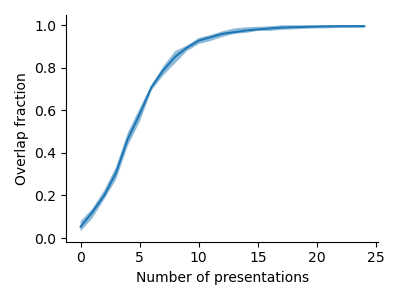}
        \caption{Recall over training the FSM}
    \end{subfigure}
    \hspace{0.05\textwidth}
    \begin{subfigure}[h]{0.4\textwidth}
        \includegraphics[width=\textwidth]{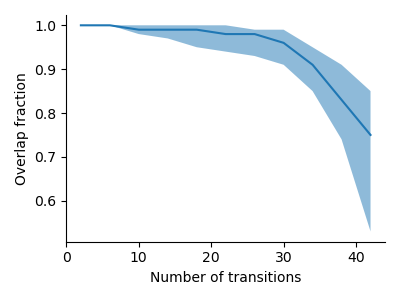}
        \caption{Performance versus size of FSM}
    \end{subfigure}
    \caption{In (a), we train the model by repeatedly presenting each transition of the FSM in Figure \ref{fig:fsm} and measure performance after a given number of presentations. In (b), we train the model using FSMs of different sizes and again measure the recall. Here, $n=5000, k=70, p=0.4, \beta=0.1$ and the model is trained with 15 presentations of each transition.    
    Dark center line indicates mean over 10 trials; shaded area indicates the range.}
    \label{fig:fsm_emp}
\end{figure}

\paragraph{Effect of parameters on FSM memorization.} We examine the effect of area size (Figure \ref{fig:fsmvary}, left) and edge density (Figure \ref{fig:fsmvary}, right) on the success of FSM memorization. Using the FSM simulated in Figure 4 (see Figure \ref{fig:fsm_example} for a diagram), we vary the relevant parameter while training the model to memorize each transition of this FSM. We measure performance through the fraction of neurons in the appropriate next state assembly which fire two rounds after the assemblies corresponding to a given state/symbol pair are made to fire (\emph{recall}), with the minimum taken over all transitions, and the fraction of a random sample of length $20$ input strings which are correctly accepted or rejected (\emph{classification accuracy}), in the sense that the state assembly with the most neurons firing on the last round corresponds to the correct terminal state. Our analytical results (where again the probability of success increases with $n$, and we require the product $kp$ to be sufficiently large) predict that increasing either $n$ or $p$ should improve performance (i.e. increase both recall and classification accuracy) which is observed experimentally. Notably, classification accuracy reaches a perfect score much sooner than recall.

\begin{figure}[h!]
    \centering
    \includegraphics[width=0.45\linewidth]{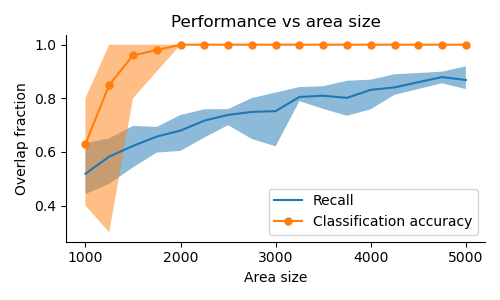}
    \includegraphics[width=0.45\linewidth]{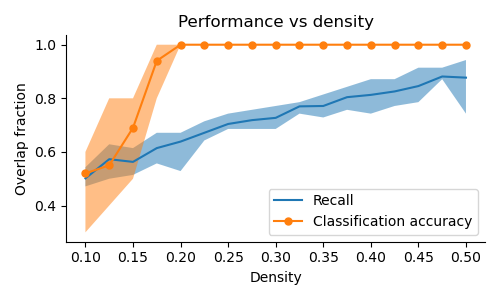}
    \caption{We measure performance of FSM memorization as two key parameters of the model (area size and density of connections) are varied.  On the left, we vary $n$ from $1000$ to $5000$ while $k=\sqrt{n}$, $p=0.5, \beta = 0.1$. On the right, we vary $p$ from $0.1$ to $0.5$ while $n = 5000, k = 70, \beta = 0.1$. In both cases the FSM is given in Figure \ref{fig:fsm_example} (with 11 symbols and 3 non-terminal states) and training consists of 15 presentations of each transition. Dark center line is the average over 10 trials, while shaded area is the range.}
    \label{fig:fsmvary}
\end{figure}

\paragraph{Effect of input string length on FSM simulation.} We examine the effect of the length of the input string on the accuracy of the simulation of the FSM on it (Figure \ref{fig:fsmlen}). For a few choices of the key parameters (area size $n$ and density $p$), we train the model to memorize each transition of the FSM in Figure \ref{fig:fsm_example} and then test its performance on a random sample of input strings of various sizes. As above, we measure performance via \emph{classification accuracy}, the fraction of strings which are correctly accepted or rejected (i.e. the state assembly with the most neurons firing on the last round corresponds to the correct terminal state). For $n$ and $p$ such that all transitions are performed accurately, we expect that classification accuracy should be very high even for long strings, while for smaller $n$ and $p$ classification accuracy should decay with string length. This is largely supported by the simulations, where classification accuracy remains nearly perfect for all string lengths for the larger choices of $n$ and $p$ and decays somewhat for the smaller choices, although notably even for long strings the classification accuracy is still substantially higher than chance (which would be $0.2$ with the $5$ states of the FSM).

\begin{figure}[h!]
    \centering
    \includegraphics[width=0.45\linewidth]{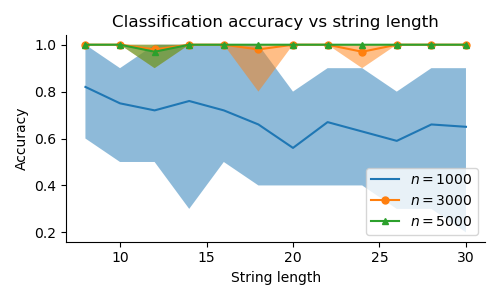}
    \includegraphics[width=0.45\linewidth]{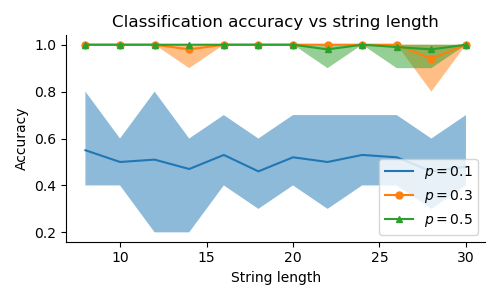}
    \caption{We measure performance of FSM simulation as the length of the input string increases, for a few choices of area size (left) and edge density (right). On the left, $k = \sqrt{n}, p=0.5, \beta=0.1$; on the right, $n=5000, k=70, \beta=0.1$. In both cases the FSM is given in Figure \ref{fig:fsm_example} (with 11 symbols and 3 non-terminal states) and training consists of 15 presentations of each transition. Dark center line is the average over 10 trials, while shaded area is the range.}
    \label{fig:fsmlen}
\end{figure}

\section{Proofs}\label{sec: proofs}

\paragraph{Proof sketches.}
All of our proofs rely on a few basic ingredients. A key observation underpinning all of them is that in the range of parameters assumed, the neurons which activate in response to a given sequence element/transition do not change over the course of training w.h.p. Intuitively, if the weights change sufficiently slowly, then repeated presentations of the same stimulus should activate the same set of neurons. We then show that the overlap of assemblies corresponding to distinct sequence elements/transitions will not be too large, which will allow them to be distinguished and ensures that increasing weights between one pair of assemblies will not interfere with others. Finally, we bound the number of rounds of training needed to ensure that recall occurs successfully, which simply involves comparing the input a neuron in the correct assembly will receive versus a neuron outside of the correct assembly, and choosing the weight to be sufficiently high so that the neurons in the correct assembly have the highest input w.h.p.

To proceed with the proofs in full, we first need a few crucial lemmas. Given a pair of inputs to a brain area, Lemma \ref{lemma:overlap} provides an upper bound (which depends on the overlap of these inputs) on the overlap of the sets of neurons which fire in response to them, when plasticity is in effect.

\begin{lemma} \label{lemma:overlap}
Let $\delta > 0$. Let $I_1, I_2$ be sets of neurons providing input to brain area $A$. (Note that $I_1, I_2$ might intersect $A$.) Suppose that $I_1$ fires, forming a cap $C_1$, with weights from $I_1$ to $C_1$ increased by $1 + \gamma$.
Some time later $I_2$ fires, forming a cap $C_2$. Suppose that $|I_1|, |I_2| \ge k$, and that $kp \ge 6\ln \frac{n}{k}$. Then if $|I_1 \cap I_2| = 0$, we have
\[\Pr\left[|C_1 \cap C_2| \ge 2 \max \{\tfrac{k^2}{n}, 3\ln \tfrac{1}{\delta}\}\right] \le \delta\]
and otherwise, for $\epsilon \ge k / n$ and
\[\gamma \le \frac{\sqrt{3|I_2|p\ln \frac{n}{k}} - \max \{\sqrt{3|I_1 \cap I_2|p\ln \frac{2n}{\delta}}, 3\ln \frac{2n}{\delta}\} - \sqrt{3|I_2 \setminus I_1|p\ln \frac{2}{\epsilon}}}{2\max \{|I_1 \cap I_2|p, 3\ln \frac{2n}{\delta}\}}\]
we have 
\[\Pr\left[|C_1 \cap C_2| \ge \max \left\{\epsilon k, 6\ln \tfrac{2}{\delta}\right\}\right] \le \delta\]
\end{lemma}

Lemma \ref{lemma:neuronlottery} controls the number of caps a given neuron will enter in response to a sequence of stimuli with bounded overlap.

\begin{lemma} \label{lemma:neuronlottery}
Let $I_1, \ldots, I_L$ be a set of inputs to a brain area $A$ (which might intersect $A$) which are presented in sequence, satisfying $|I_1| \ge k, |I_t| = m \ge k$ for all $t > 1$, $|I_s \cap I_t| \le \Delta$ for all $s \neq t$, and 
\[L \le 3 \ln n \left(\frac{n}{k}\right)^{(1 - \sqrt{\frac{\Delta \ln n}{m \ln \frac{n}{k}}})^2} \] Let $C_t$ denote the cap denoting from $I_t$, and $\1_t$ denote the indicator vector for $C_t$. 
Then for 
\[\beta \le \frac{ (1 - \sqrt{\frac{\Delta \ln n}{m \ln \frac{n}{k}}})^2 \ln \frac{n}{k} - \ln L - \frac{1}{3}}{2 (\max \{\Delta p, 6 \ln n\})^2 }\]
and $kp \ge 3 \ln n$ we have
\[\sum_{t=1}^L \1_t(i) < 6 \ln n \]
w.h.p. for all $i \in A$.
\end{lemma}

Lemma \ref{lemma:stable} provides a criterion to ensure that the same set of neurons will fire in response to the same stimulus.

\begin{lemma} \label{lemma:stable}
Let $\delta > 0$, and let $I$ be a set of neurons providing input to brain area $A$ with $|I|p \ge 3 \ln \frac{1}{\delta}$. Let $C \subseteq A$ be a set of $k$ neurons, where each neuron in $C$ has at least $d$ incoming edges from $I$, each strengthened by a factor of $1 + \gamma$. Suppose no neuron outside of $C$ has more than $m$ incoming connections from $C$ which have been strengthened, each by a factor no more than $1+c\gamma$. Then for 
\[\gamma \ge \frac{|I|p + \sqrt{3|I|p \ln \frac{n}{\delta}} - d}{d-cm}\]
if $I$ fires, the resulting cap $C'$ will equal $C$ w.p. $1 - \delta$
\end{lemma}

Lemma \ref{lemma:highoverlap} upper bounds the overlap of the caps for a pair of stimuli, at a larger overlap than Lemma \ref{lemma:overlap}.

\begin{lemma} \label{lemma:highoverlap}
Let $I_1$ and $I_2$ be two sets of neurons providing input to area $A$, with $|I_1| = |I_2| = 2k$ and $|I_1 \cap I_2| \le (1 + o(1))k$, and let $A_1, A_2$ denote their respective caps. Then
\[|A_1 \cap A_2| \le \frac{k}{2}\]
w.p. $1 - o(n^{-1})$.
\end{lemma}

Finally, Lemma \ref{lemma:expsoln} is a fact about the fixed points of a certain class of exponential functions.

\begin{lemma} \label{lemma:expsoln}
Let $f(x) = \exp(-a(b(1-x)^2-1))$ for $a > 0, b \ge 1 + 2a^{-1}$. Then there exists some $x^* \in \R$ such that $f(x^*) = x^*$ and moreover
\[x^* \le \exp(-a(b-1))\]
\end{lemma}

We can now proceed with the proofs of the theorems.

\begin{proof}[Proof of Theorem \ref{thm:simple}]
We will first show that on the first presentation, the overlap between the sets of neurons which fire for different elements of the sequence will not be too large for sufficiently small plasticity. Using this fact, we can then show that the sequence of neurons which fire will remain the same over repeated presentations of the stimulus sequence. Finally, this second claim makes it simple to show that after enough presentations the sequence of neurons will be strongly linked enough that when only the first cap fires, the rest of the sequence will as well without external stimulus.

We will prove the first claim by induction on $\sigma$. The base case $\sigma = 1$ holds vacuously. For $\sigma \ge 1$, we note that if $|A_{\rho}(1) \cap A_{\sigma}(1)| \le 6 \ln \frac{n}{k}$ for all $\rho < \sigma$, then the input to $A_{\sigma+1}(1)$ ($S_{\sigma+1}$ and $A_\sigma(1)$) overlaps the input to any other set $A_{\rho}(1)$ by at most $\Delta + 6 \ln \frac{n}{k}$. As weights are increased by $1 + \beta$ and
\[\beta \le \frac{\ln \frac{n}{kL} - \frac{1}{3}}{2 (\max \{\Delta p, 6 \ln n\})^2 } \le \frac{\sqrt{6kp \ln \frac{n}{k}} - \max\{\sqrt{6 \Delta p \ln Ln}, 6\ln nL\} - \sqrt{6kp \ln k}}{2\max \{\Delta p, 6 \ln \frac{n}{k}\} }\]
by Lemma \ref{lemma:overlap} we have $|A_{\sigma+1}(1) \cap A_{\rho}(1)| \le 6\ln \frac{n}{k}$ w.p. $1 - o(L^{-2})$ as well. A union bound over all $O(L^2) $ events shows that the probability that $|A_{\sigma}(1) \cap A_{\rho}(1)| > 6\ln \frac{n}{k}$ for any pair $\rho \neq \sigma$ is $o(1)$.

We will now show that $A_\sigma(t) = A_\sigma(1)$ for all $\sigma$, by first proving that $A_\sigma(2) = A_\sigma(1)$ via induction on $\sigma$. First note that, by Lemma \ref{lemma:neuronlottery}, no neuron in $A$ makes more than $6 \ln n$ caps $A_1(1), \ldots, A_L(1)$ w.h.p. Let $X_i(t, \sigma)$ denote the input to $i$ on round $t$ during presentation of $S_\sigma$. For the base case $\sigma = 1$, for $i \in A_{1}(1)$, we have
\[e(S_1, i) \ge kp + \sqrt{3kp \ln \frac{n}{k}}\]
and so its input is
\[X_i(2, 1) \ge \frac{(1 + \beta)(kp + \sqrt{3kp \ln \frac{n}{k}})}{2np + 2\beta kp}\]
while for $j \not\in A_1(1)$,
\begin{align*}
    X_j(2, 1) &\le \frac{kp + \sqrt{3kp\ln \frac{n}{k}} + (1 + \beta)^{6\ln n}(\Delta p + \max \{\Delta p, 3 \ln n\})}{2np + (1 + \beta)^{6\ln n}(\Delta p + \max \{\Delta p, 3 \ln n\})}\\
    &\le \frac{kp + \sqrt{3kp\ln \frac{n}{k}} + 2(1 + \beta)^{6\ln n}\max \{\Delta p, 3 \ln n\}}{2np + (1 + \beta)^{6\ln n}(\Delta p + \max \{\Delta p, 3 \ln n\})}
\end{align*}
Observe that for
\[\beta \le \le \frac{ \ln \frac{n}{kL} - \frac{1}{3}}{2 (\max \{\Delta p, 6 \ln n\})^2 } \le \frac{1}{6\ln n}\]
we have
\[X_j(2, 1) \le \frac{(1 + \beta)(kp + \sqrt{3kp \ln \frac{n}{k}})}{2np + 2\beta kp}\]
in which case $X_i(2, 1) > X_j(2, 1)$ for all $i \in A_1(1), j \not\in A_1(1)$.

Now assume that $\sigma > 1$ and $A_{\rho}(2) = A_{\rho}(1)$ for all $\rho < \sigma$. For $i \in A_\sigma(1)$, 
\[e(S_\sigma \cup A_{\sigma-1}(1), i) \ge 2kp + \sqrt{6kp \ln \frac{n}{k}}\]
and so
\[X_i(2, \sigma) \ge \frac{(1 + \beta) (2kp + \sqrt{6kp \ln \frac{n}{k}})}{np + \beta 2kp + \sqrt{6kp \ln \frac{n}{k}}}\]
 On the other hand, for all $j \not\in A_\sigma(1)$, we have
\[X_j(2, \sigma) \le \frac{2kp + \sqrt{6kp \ln \frac{n}{k}} + 4(1 + \beta)^{6 \ln n}\max \{\Delta p, 3 \ln n\}}{np + 4(1 + \beta)^{6 \ln n}\max \{\Delta p, 3 \ln n\}}\]
For 
\[\beta \le \frac{ \ln \frac{n}{kL} - \frac{1}{3}}{2 (\max \{\Delta p, 6 \ln n\})^2 } \le \frac{1}{12 \ln n} \] 
we have
\[X_j(2, \sigma) \le \frac{(1 + \beta) (2kp + \sqrt{6kp \ln \frac{n}{k}})}{np + \beta 2kp + \sqrt{6kp \ln \frac{n}{k}}}\]
and so $X_i(2, \sigma) > X_j(2, \sigma)$ for all $i \in A_\sigma(1), j \not\in A_\sigma(1)$, which means $A_\sigma(2) = A_\sigma(1)$.

Now, we use double induction on $t$ and $\sigma$, with $t = 2$ as the base case. Let $X_i(\sigma, t)$ denote the input to neuron $i$ when $S_\sigma$ fires on round $t$. Suppose that $A_{\rho}(s) = A_{\rho}(1)$ for all $\rho$ and $s < t$ and $A_{\rho}(t) = A_{\rho}(1)$ for all $\rho < \sigma$. For $i \in A_\sigma(1)$, we have
\[X_i(t, \sigma) \ge \frac{(1 + \beta) X_i(t-1, \sigma)}{2np + \beta X_i(t-1, \sigma)}\]
For $j \not\in A_\sigma(1)$,
\[X_j(t, \sigma) \le \frac{X_j(t-1, \sigma) + (e^{6\beta \ln n} - 1)(X_j(t-1, \sigma) - 2kp - \sqrt{6 kp \ln \frac{n}{k}})}{2np + (e^{6\beta \ln n} - 1) (X_j(t-1, \sigma) - 2kp - \sqrt{6 kp \ln \frac{n}{k}})}\]
For 
\[\beta \le \frac{ \ln \frac{n}{kL} - \frac{1}{3}}{2 (\max \{\Delta p, 6 \ln n\})^2 } \le \frac{(1 - \sqrt{\frac{\Delta \ln n}{k\ln \frac{n}{k}}})^2 \ln \frac{n}{k} - \ln L - \frac{1}{3}}{2 (\max \{\Delta p, 6 \ln n\})^2 }\]
and $\Delta \le k / \ln^2 n$, we have $X_j(t, \sigma) \le X_i(t, \sigma)$ and hence, $A_{\sigma}(t) = A_\sigma(1)$.

Now, we will prove the last claim. For $\sigma = 1$, the argument is the same as during training since $S_1$ still fires, so we have $\widehat A_\sigma = A_\sigma(1)$. Suppose that $|\widehat A_{\sigma - 1} \cap A_{\sigma-1}(1)| = (1-\epsilon) k$. Then for $i \in A_\sigma(1)$, 
\[X_i(\sigma) \ge \frac{(1 - \epsilon) (1 + \beta)^T (kp + \sqrt{\frac{3}{2}kp \ln \frac{n}{k}}) + \epsilon kp}{((1 + \beta)^T - 1)(kp + \sqrt{3kp \ln \frac{n}{k}}) + 2np} = \tau_\epsilon\]
while for $j \not\in A_\sigma(1)$,
\[X_j(\sigma) \sim \frac{\mathcal B(k, p)}{2np}\]
Hence,
\begin{align*}
    \Pr[j \in \widehat A_\sigma &\cond j \not\in A_\sigma(1)] \le \Pr\left[X_j(\sigma) \ge \tau_\epsilon \right]\\
    &\le \exp\left(-\frac{1}{3kp}\left((1 - \epsilon)\beta T(kp + \sqrt{\frac{3}{2}kp \ln \frac{n}{k}}) + (1 - \epsilon) \sqrt{\frac{3}{2}kp \ln \frac{n}{k}}\right)^2\right)\\
    &\le \exp\left(- \frac{1}{2} (1 - 2\epsilon) (1 + 2\beta T)^2 \ln \frac{n}{k}\right)
\end{align*} The fraction of newcomers is hence, in expectation, no more than
\[\frac{n}{k}\exp\left(- \frac{1}{2}(1 - 2\epsilon) (1 + 2\beta T)^2 \ln \frac{n}{k}\right)\]
We seek a fixed point $\epsilon^*$ of this function. By Lemma \ref{lemma:expsoln}, we have
\[\epsilon^* \le \left(\frac{k}{n}\right)^{2\beta T}\]
as long as
\[T \ge \frac{1}{\beta \ln \frac{n}{k}}\]
Now, for perfect recall, suppose that
\[T  \ge \frac{1}{\beta}\sqrt{\frac{\ln nL}{2\ln \frac{n}{k}}}\]
Then
\[\exp\left(- \frac{1}{2} (1 + 2\beta T)^2 \ln \frac{n}{k}\right) \le \frac{1}{nL}\]
Then a union bound over vertices and $1 \le \sigma \le L$ shows that the probability $\widehat A_\sigma \neq A_\sigma(1)$ for some $\sigma$ is $o(1)$.  
\end{proof}

\begin{proof}[Proof of Theorem \ref{thm:scaffold}]
The argument proceeds similarly as in the proof of Theorem 1: We bound the overlap of different caps during the first presentation, show that the sequence of caps will remain the same over subsequent presentations, and show that after enough presentations the sequence of caps will be recalled without external input after the first round. Crucially, a larger fraction of the input to the correct neurons in the main area will still fire in the absence of external stimuli, which allows memorization to occur more quickly than in the absence of the auxiliary area.

We will prove the first claim by induction on $\sigma$. The base case $\sigma = 1$ holds vacuously. For $\sigma \ge 1$, we note that if $|A_{\rho}(1) \cap A_{\sigma}(1)| \le 6\ln \frac{n}{k}$ for all $\rho < \sigma$, and $|B_{\rho}(1) \cap B_{\sigma}(1)| \le 6\frac{n}{k}$ for all $\rho < \sigma - 1$, then the input to $A_{\sigma+1}(1)$ ($S_{\sigma+1}$, $A_\sigma(1)$, and $B_{\sigma-1}(1)$) overlaps the input to any other set $A_{\rho}(1)$ by at most $\Delta + 12 \ln \frac{n}{k}$. As weights are increased by $1 + \beta$ and
\[\beta \le \frac{\ln \frac{n}{kL} - \frac{1}{3}}{2 (\max \{\Delta p, 6 \ln n\})^2 } \le \frac{\sqrt{9kp \ln \frac{n}{k}} - \max\{\sqrt{9 \Delta p \ln Ln}, 3\ln nL\} - \sqrt{9kp \ln k}}{3\max \{\Delta p, 6 \ln \frac{n}{k}\} }\]
by Lemma \ref{lemma:overlap} we have $|A_{\sigma+1}(1) \cap A_{\rho}(1)| \le 6\ln \frac{n}{k}$ w.p. $1 - o(L^{-2})$ as well. The argument proceeds similarly for $B_\sigma(1)$. Then a union bound over all $O(L^2)$ events shows that the probability that $|A_{\rho}(1) \cap A_{\sigma}(1)|$ or $|B_{\rho}(1) \cap B_{\rho}(1)| > 6\ln \frac{n}{k}$ for any pair $\sigma \neq \rho$ is $o(1)$.

We will now show that $A_\sigma(t) = A_\sigma(1)$ and $B_\sigma(t) = B_\sigma(1)$ for all $\sigma$. We note that, by Theorem 1, $B_\sigma(t) = B_\sigma(1)$ for all $\sigma$ and $t$ w.h.p. as long as $A_\rho(s) = A_\rho(1)$ for all pairs $(\rho, s)$ where either $s < t$ or $\rho < \sigma$. So, it will suffice to only show that $A_\sigma(s) = A_\sigma(1)$, which we will accomplish by first proving that $A_\sigma(2) = A_\sigma(1)$ via induction on $\sigma$. First note that, by Lemma \ref{lemma:neuronlottery}, no neuron in $A$ makes more than $6 \ln n$ caps $A_1(1), \ldots, A_L(1)$. Let $X_i(t, \sigma)$ denote the input to $i$ on round $t$ during presentation of $S_\sigma$. For the base case $\sigma = 1$, for $i \in A_{1}(1)$, we have
\[e(S_1, i) \ge kp + \sqrt{3kp \ln \frac{n}{k}}\]
and so its input is
\[X_i(2, 1) \ge \frac{(1 + \beta)(kp + \sqrt{3kp \ln \frac{n}{k}})}{2np + 2\beta kp}\]
while for $j \in A \setminus A_1(1)$,
\begin{align*}
    X_j(2, 1) &\le \frac{kp + \sqrt{3kp\ln \frac{n}{k}} + (1 + \beta)^{6\ln n}(\Delta p + \max \{\Delta p, 3 \ln n\})}{2np + (1 + \beta)^{6\ln n}(\Delta p + \max \{\Delta p, 3 \ln n\})}\\
    &\le \frac{kp + \sqrt{3kp\ln \frac{n}{k}} + 2(1 + \beta)^{6\ln n}\max \{\Delta p, 3 \ln n\}}{2np + (1 + \beta)^{6\ln n}(\Delta p + \max \{\Delta p, 3 \ln n\})}
\end{align*}
Observe that for
\[\beta \le \le \frac{ \ln \frac{n}{kL} - \frac{1}{3}}{2 (\max \{\Delta p, 6 \ln n\})^2 } \le \frac{1}{6\ln n}\]
we have
\[X_j(2, 1) \le \frac{(1 + \beta)(kp + \sqrt{3kp \ln \frac{n}{k}})}{2np + 2\beta kp}\]
in which case $X_i(2, 1) > X_j(2, 1)$ for all $i \in A_1(1), j \in A \setminus A_1(1)$. Hence, $A_2(1) = A_1(1)$.

Now let $\sigma > 1$, and assume that $A_{\rho}(2) = A_{\rho}(1)$ and $B_{\rho - 1}(2) = B_{\rho - 1}(1)$ for all $\rho < \sigma$. For $i \in B_{\sigma-1}(1)$, 
\[e(A_{\sigma-1} \cup B_{\sigma-2}(1), i) \ge 2kp + \sqrt{6kp \ln \frac{n}{k}}\]
and so
\[X_i(2, \sigma) \ge \frac{(1 + \beta) (2kp + \sqrt{6kp \ln \frac{n}{k}})}{2np + \beta 2kp + \sqrt{6kp \ln \frac{n}{k}}}\]
 On the other hand, for all $j \in B \setminus B_{\sigma-1}(1)$, we have
\[X_j(2, \sigma) \le \frac{2kp + \sqrt{6kp \ln \frac{n}{k}} + 12(1 + \beta)^{6 \ln n} \ln \frac{n}{k}}{2np + 12(1 + \beta)^{6 \ln n}\ln n}\]
For 
\[\beta \le \frac{ \ln \frac{n}{kL} - \frac{1}{3}}{2 (\max \{\Delta p, 6 \ln n\})^2 } \le \frac{1}{12 \ln n} \], we have
\[X_j(2, \sigma-1) \le \frac{(1 + \beta) (2kp + \sqrt{6kp \ln \frac{n}{k}})}{2np + \beta 2kp + \sqrt{6kp \ln \frac{n}{k}}}\]
and so $X_i(2, \sigma) > X_j(2, \sigma)$ for all $i \in B_{\sigma-1}(1), j \in B \setminus B_{\sigma-1}(1)$, which means $B_{\sigma-1}(2) = B_{\sigma-1}(1)$
Then, for $i \in A_{\sigma}(1)$, 
\[e(S_{\sigma} \cup A_{\sigma}(1) \cup B_{\sigma-1}(1), i ) \ge 3kp + \sqrt{9kp \ln \frac{n}{k}}\] 
and so
\[X_i(2, \sigma) \ge \frac{(1 + \beta) (3kp + \sqrt{9kp \ln \frac{n}{k}})}{3np + \beta 3kp + \sqrt{9kp \ln \frac{n}{k}}}\]
 On the other hand, for all $j \in A \setminus A_{\sigma}(1)$, we have
\[X_j(2, \sigma) \le \frac{3kp + \sqrt{6kp \ln \frac{n}{k}} + 12(1 + \beta)^{6 \ln n} \ln \frac{n}{k}}{3np + 12(1 + \beta)^{6 \ln n}\ln n}\]
For 
\[\beta \le \frac{ \ln \frac{n}{kL} - \frac{1}{3}}{2 (\max \{\Delta p, 6 \ln n\})^2 } \le \frac{1}{12 \ln n} \]
we have
\[X_j(2, \sigma) \le \frac{(1 + \beta) (3kp + \sqrt{9kp \ln \frac{n}{k}})}{3np + \beta 3kp + \sqrt{9kp \ln \frac{n}{k}}}\]
and so $X_i(2, \sigma) > X_j(2, \sigma)$ for all $i \in A_{\sigma}(1), j \in A \setminus A_{\sigma}(1)$, so $A_{\sigma}(2) = A_{\sigma}(1)$. Then using induction, we have $A_\sigma(2) = A_\sigma(1)$ and $B_\sigma(2) = B_\sigma(1)$ for all $1 \le \sigma \le L$.

Now, we use double induction on $t$ and $\sigma$, with $t = 2$ as the base case. Let $X_i(t, \sigma)$ denote the input to neuron $i$ when $S_\sigma$ fires on round $t$. Suppose that $A_{\rho}(s) = A_{\rho}(1)$ and $B_{\rho-1}(s) = B_{\rho-1}(1)$ for all $\rho$ and $s < t$ and $A_{\rho}(t) = A_{\rho}(1)$ and $B_{\rho - 1}(t) = B_{\rho-1}(1)$ for all $\rho < \sigma$. For $i \in B_{\sigma-1}(1)$, we have
\[X_i(t, \sigma) \ge \frac{(1 + \beta) X_i(t-1, \sigma)}{2np + \beta X_i(t-1, \sigma)}\]
For $j \in B \setminus B_{\sigma-1}(1)$,
\[X_j(t, \sigma) \le \frac{X_j(t-1, \sigma) + (e^{6\beta \ln n} - 1)(X_j(t) - 2kp - \sqrt{6 kp \ln \frac{n}{k}})}{2np + (e^{6\beta \ln n} - 1) (X_j(t-1, \sigma) - 2kp - \sqrt{6 kp \ln \frac{n}{k}})}\]
For 
\[\beta \le \frac{ \ln \frac{n}{kL} - \frac{1}{3}}{2 (\max \{\Delta p, 6 \ln n\})^2 } \le \frac{(1 - \sqrt{\frac{\Delta \ln n}{k\ln \frac{n}{k}}})^2 \ln \frac{n}{k} - \ln L - \frac{1}{3}}{2 (\max \{\Delta p, 6 \ln n\})^2 }\]
and $\Delta \le k / \ln^2 n$, we have $X_j(t, \sigma) \le X_i(t, \sigma)$ and hence, $B_{\sigma-1}(t) = B_{\sigma-1}(1)$.

Now, we will prove the last claim. For $\sigma = 1$, the argument is the same as during training since $S_1$ still fires, so we have $\widehat A_\sigma = A_\sigma(1)$. Suppose that $|\widehat A_{\sigma - 1} \cap A_{\sigma-1}(1)| = (1-\epsilon) k$. Then for $i \in A_\sigma(1)$, 
\[X_i(\sigma) \ge \frac{(1 - \epsilon) (1 + \beta)^T (kp + \sqrt{\frac{3}{2}kp \ln \frac{n}{k}}) + \epsilon kp}{((1 + \beta)^T - 1)(kp + \sqrt{3kp \ln \frac{n}{k}}) + 2np} = \tau_\epsilon\]
while for $j \not\in A_\sigma(1)$,
\[X_j(\sigma) \sim \frac{\mathcal B(k, p)}{2np}\]
Hence,
\begin{align*}
    \Pr[j \in \widehat A_\sigma &\cond j \not\in A_\sigma(1)] \le \Pr\left[X_j(\sigma) \ge \tau_\epsilon \right]\\
    &\le \exp\left(-\frac{1}{3kp}\left((1 - \epsilon)\beta T(kp + \sqrt{\frac{3}{2}kp \ln \frac{n}{k}}) + (1 - \epsilon) \sqrt{\frac{3}{2}kp \ln \frac{n}{k}}\right)^2\right)\\
    &\le \exp\left(- \frac{1}{2} (1 - 2\epsilon) (1 + 2\beta T)^2 \ln \frac{n}{k}\right)
\end{align*} The fraction of newcomers is hence, in expectation, no more than
\[\frac{n}{k}\exp\left(- \frac{1}{2}(1 - 2\epsilon) (1 + 2\beta T)^2 \ln \frac{n}{k}\right)\]
We seek a fixed point $\epsilon^*$ of this function. By Lemma \ref{lemma:expsoln}, we have
\[\epsilon^* \le \left(\frac{k}{n}\right)^{2\beta T}\]
as long as
\[T \ge \frac{1}{\beta \ln \frac{n}{k}}\]
Now, for perfect recall, suppose that
\[T  \ge \frac{1}{\beta}\sqrt{\frac{\ln nL}{2\ln \frac{n}{k}}}\]
Then
\[\exp\left(- \frac{1}{2} (1 + 2\beta T)^2 \ln \frac{n}{k}\right) \le \frac{1}{nL}\]
Then a union bound over vertices and $1 \le \sigma \le L$ shows that the probability $\widehat A_\sigma \neq A_\sigma(1)$ for some $\sigma$ is $o(1)$.
\end{proof}

\begin{proof}[Proof of Theorem \ref{thm:fsm}]
We first bound the overlap of the arc assemblies associated with different transitions. We then show that after sufficiently many presentations of each transitions, the connections between pairs of state and symbol assemblies and their associated arc assemblies, and between arc assemblies and the correct next state assemblies, will be sufficiently strengthened so that firing any state/symbol pair will cause the appropriate arc assembly, and then in turn the appropriate next state assembly, to fire in their entirety. Once this property is established for every transition, it follows immediately that the FSM can be simulated on an arbitrary input string.

Let $A_{q, \sigma} \subseteq A$ denote the set of neurons which fires in response to $S_q \cup I_\sigma$. We need to show that after $T$ training presentations, if $S_q$ and $I_\sigma$ fire, then $S_{\delta(q, \sigma)}$ will fire two time steps later w.p. $1 - o(|Q|^{-1}|\Sigma|^{-1})$. 

We will first bound the overlap of the sets $A_{q, \sigma}$. Suppose that $(q, \sigma) \neq (q', \sigma')$. Then
\[|(S_q \cup I_\sigma) \cap (S_{q'} \cup I_{\sigma'})| \le k + \Delta\]
With $\Delta = o(k)$, by Lemma \ref{lemma:highoverlap} we have $|A_{q, \sigma} \cap A_{q', \sigma'}| \le k/2$ w.p. at least $1 - o(n^2)$. Then a union bound over all $O(|Q||\Sigma|) = O(n)$ pairs $(q, \sigma) \neq (q', \sigma')$ shows that the bound holds for every pair w.p. $1 - o(1)$. 

Now, after training, it is clear that $A_{q, \sigma}$ will fire in response to $S_q$ and $I_\sigma$. It remains to ensure that all of $S_{\delta(q, \sigma)}$ will fire. With probability $1 - o(1)$, we have
\[e(A_{q, \sigma}, i) \ge kp - \sqrt{2kp \ln n}\]
for all $i \in S_{\delta(q, \sigma)}$ and $q \in Q, \sigma \in \Sigma$. On the other hand, for any $j \in S$,
\[e(A_{q, \sigma} \cap A_{q', \sigma'}, j) \le \frac{kp}{2} + \sqrt{\frac{3}{2}kp \ln n}\]
Let
\[X_j = e(A_{q, \sigma} \setminus A_{q', \sigma'}, j)\]
and note that $X_j$ is the sum of i.i.d. Bernoulli random variables with $\E X_j \le kp / 2$. Let $S_{q, \sigma}'$ denote the cap which fires in response to $A_{q, \sigma}$. Then
\begin{align*}
    \Pr(j \in S_{q, \sigma}' \cond j \not\in S_{\delta(q, \sigma)}) &\le 
    \Pr\left(X_j \ge (1 + \beta)^T\left(\frac{kp}{2} - \sqrt{\frac{7}{2}kp \ln n}\right)\right)\\
    &\le \exp\left(-\frac{2(T\beta \frac{kp}{2} - (1 + T \beta) \sqrt{\frac{7}{2}kp \ln n})^2}{3kp}\right)
\end{align*} For 
\[T \ge \frac{12}{\beta} \sqrt{\frac{\ln n}{kp}}\]
we have
\[\Pr(j \in S_{q, \sigma}' \cond j \not\in S_{\delta(q, \sigma)}) \le \frac{1}{n^2}\]
and hence via a union bound we will have $S_{q, \sigma}' = S_{\delta(q, \sigma)}$ for every $q \in Q, \sigma \in \Sigma$ w.p. $1 - o(1)$.
\end{proof}

\begin{proof}[Proof of Lemma \ref{lemma:stack}]
Suppose the content of the tape after $L$ operations is $\sigma_1, \ldots, \sigma_{\tilde L}$. We will show that with probability $1 - o(1)$, a set $A_1, \ldots, A_{\tilde L}$ of assemblies will be created, where $A_1 \subseteq H_i$ and $A_{t} \subseteq H_{i + t \pmod 3}$. This set has the properties that for any $1 \le t \le \tilde L$, (i) if $A_t$ fires then $S_{\sigma_t}$ will fire on the next round, and (ii) if $A_t$ fires and $H_{i+t+1 \pmod 3}$ is disinhibited then $A_{t+1}$ will fire on the next round. 

We proceed by induction on $L$, the number of operations. Immediately, if the $(L+1)$th operation is ``delete'', then after $T$ rounds $H_{i+1 \pmod 3}$ will be disinhibited, and by the inductive hypothesis then $A_{2}, \ldots, A_{\tilde L}$ is a set of assemblies with the required property. So, we consider only ``add'' operations henceforth. We will show that each ``add'' operation creates an assembly $A'$ which will fire $A_1$ and the associated symbol assembly with probability $1 - o(L^{-1})$. As there are at most $L$ operations, a union bound over all operations will show the entire sequence succeeds w.p. $1 - o(1)$.

For the first ``add'' operation, there is no next symbol, so the set of neurons $A_1$ which fires randomly only needs to be linked to the correct symbol, $S_{\sigma_1}$. By Theorem 3 of \citet{papadimitriou2019random}, for $\beta \ge \sqrt{\ln n / kp}$ $A_1$ will stabilize after $Z_1$ fires at least
\[T' = \frac{1}{\beta} \frac{\ln k}{\sqrt{kp}}\]
times. Every neuron in $S_{\sigma_1}$ will have at least $kp - \sqrt{2kp \ln n}$ connections from $A_1$, while no neuron outside has more than $kp + \sqrt{3kp\ln n}$. Hence, it suffices that
\[(1 + \beta)^{T - T'} \left(kp - \sqrt{2kp \ln n}\right) \ge kp + \sqrt{3kp \ln n}\]
For
\[T \ge 5\frac{\sqrt{\ln n}}{\beta}\]
this bound holds. 

Now, suppose that the first $t$ operations were successfully simulated, for $t \ge 1$. Hence, on rounds $tT, tT+1, \ldots, tT + T - 1$, some area (say $H_i$) fired, while $C_{\text{delete}}$ will fire on round $tT + T- 1$. For a ``delete'' operation, this implies immediately that $H_{i+1}$ will begin to fire on round $(t+1)T$ and continue firing through round $(t+1)T + T - 1$. Moreover, to ensure that the correct set of neurons fires (say $A_{t+1} \subset H_{i+1}$) on round $(t+1)T$, we need that every neuron of $A_{t+1}$ receives more input than every neuron in $H_{i+1} \setminus A_{t+1}$. Once again, $T \ge k + 10/\beta$ will ensure this holds. Assuming $A_{t+1}$ fires, every neuron of $S_{\sigma_{t+1}}$ has its incoming weights from $A_{t+1}$ strengthened by $(1 + \beta)^{T - \ln k}$, so $S_{\sigma_{t+1}}$ will also fire under the same conditions on $T$.

On the other hand, suppose the $(t+1)$'th operation adds a symbol to the beginning of the tape. By induction, some area (say $H_i$) fired on rounds $tT, tT+1, \ldots, (t+1)T - 1$, while $C_{\text{add}}$ fires on round $(t+1)T - 1$. So, $H_{i-1}$ will begin firing on $(t+1)T$ (along with $H_i$, which will continue to fire) and continue firing until $(t+2)T - 1$. Simultaneously, on rounds $(t+1)T, (t+1)T + 1, \ldots, (t+1)T + T - 1$ the assemblies $Z_{t+1}$, $S_{\sigma_{t+1}}$ fire together. Let $A_{t+1}$ denote the set of neurons that fire in response to stimulus from $E_1 \cup E_2$. Then for
\[\beta T \le \frac{\sqrt{3k} - \sqrt{36 \ln n}}{12 \sqrt{\ln n p}}\]
by Lemma \ref{lemma:overlap} w.p. $o(L^{-1})$, we will have $|A_{t+1} \cap A_s| \le 6 \ln n$. By Lemma \ref{lemma:neuronlottery}, no neuron enters more than $6 \ln n$ assemblies.

It remains to show that the correct assembly $A_t \subseteq H_i$ will continue to fire for the next $T$ rounds. As shown above, this will be sufficient to ensure that when $A_{t+1}$ is firing and $H_i$ is disinhibited, that $A_t$ will begin firing. A neuron in $A_t$ has fired alongside all other neurons in $A_t$ at least $T$ times, so its input is at least
\[\left(3 + \beta T\right) \left(kp - \sqrt{2kp \ln n}\right)\]
On the other hand any neuron outside of $A_t$ receives at most
\[3kp + \sqrt{9kp \ln n} + 72 \beta T \ln^2 n\]
input from $A_{t+1}, A_t,$ and $Z_{t+1}$ combined w.p. $1 - o(L^{-1})$. We then only need that
\[\beta T \left(kp - \sqrt{2kp \ln n} - 72 \ln^2 n p\right) \ge \sqrt{9kp \ln n}\]
As for $kp \ge 6\ln n$ we have
\[5\frac{\sqrt{\ln n}}{\beta} \ge \frac{1}{\beta}\frac{\sqrt{9kp \ln n}}{kp - \sqrt{2kp \ln n} - 12 \ln^2 n p}\]
this is satisfied.

\end{proof}

\begin{proof}[Proof of Theorem \ref{thm:turing}]
    The simulation consists of three areas $H_i^L, H_i^R$ for each of the ``right'' and ``left'' halves of the tape, a control area $D$ for both halves of the tape jointly, and the three areas $I, S, A$ for the FSM. The outline of its operation is as follows: The FSM simulation proceeds as in Theorem 4, modified so that if $\delta(q, \sigma) = (p, \rho, d)$, the assembly $A_{q, \sigma}$ causes $S_p$, $I_\rho$, and $D_d$ to fire, where $d \in \{L, R\}$. The interneurons of the FSM simulation now allow their respective areas to fire after $T$ rounds, where $T$ is the number of repetitions required by the tape simulation. The tape halves are configured as in Lemma \ref{lemma:stack}, where the firing of $D_L$ signals a deletion for the left half and a addition of the top symbol from the left half for the right stack. The new assembly on the right half becomes linked to $I_\sigma$ by their concurrent firing, which effects a leftward movement of the tape head (and similarly for $D_R$).

    By Theorem 4, the FSM simulation in areas $I, S, A$ will simulate the transition function of the Turing machine w.h.p., while by Lemma \ref{lemma:stack} each of the stack simulations will perform as desired w.h.p. which simulates the TM's tape. We then have three events which each occur w.h.p., so their intersection does as well.
\end{proof}

\section{Discussion}
\name\ is a simple mathematical model of the brain involving random connectivity, Hebbian plasticity, local inhibition, and long-range interneurons, which provably captures key aspects of cognition involving sequences.  Here we demonstrated that sequences can be memorized and copied in various modes, and furthermore that, through sequences, our model can learn to recognize patterns modeled by finite state machines; a more sophisticated use of sequences of assemblies is capable of universal computation.  While this last point is of course primarily of theoretical interest, it is a rather useful point to make about the power of \name\ as a model of brain computation.   Brain-like computation can happen with no explicit control flow or commands --- that is to say, no {\em program.}  All that is needed is biologically plausible neural hardware and the presentation of stimuli.

A striking difference in the behavior of the brain and silicon computers is in their mechanisms for the creation and recall of memories. This is the starting point of the assembly model, continued here. A very interesting and unexpected discovery is rigorous mathematical evidence for the emergence of the utility of mnemonics, memory palaces and other memory aids \citep{maguire2003routes}, a uniquely brain-centered phenomenon that has no analogy for computers.  

Although we have explored how the brain might encode and memorize sequences here, a crucial gap remains: How the brain generates sequences. Augmenting our FSM simulation with probabilistic transitions would allow it to generate random sequences, effectively capturing the brain's ability to sample from a {\em probabilistic} finite automaton; this approach could be further generalized to allow sampling from a graphical model. How probabilistic transitions might be realized in \name, and moreover learned from streams of stimuli, is an important and urgent future direction for this work, and would be the basis of emergent statistical computation in the brain.

Finally, we must note that, in our discussion of the way in which sequences are memorized and recalled in the brain, we have not yet mentioned {\em language,} arguably the most remarkable faculty related to sequences that human brains enjoy; see \citet{mitropolsky2021biologically, mitropolsky2022center} for work on assemblies and language.

\acks{MD and SV are supported in part by NSF awards CCF-1909756, CCF-2007443 and CCF-2134105, and a NSF Graduate Research Fellowship. CP is supported by NSF Awards CCF-1763970 and CCF-1910700, and a research contract with Softbank.}

\bibliography{references}


\end{document}